%% file: MS-BSP.tex
\documentclass[journal]{IEEEtran}
\input{packages}
\input{commands}

\IEEEoverridecommandlockouts                              % This command is only needed if \overrideIEEEmargins                                      % Needed to meet printer requirements.

\title{\LARGE \bf
Measurement Simplification in $\rho$-POMDP with Performance Guarantees
}

\author{Tom Yotam and Vadim Indelman % <-this % stops a space
\thanks{Tom Yotam is with the Facaulty of Mathematics, Technion - Israel Institute of Technology, Haifa 32000,
	Israel, {\tt tomyo@campus.technion.ac.il}. Vadim Indelman is with the Department of Aerospace Engineering, Technion - Israel Institute of Technology, Haifa 32000, Israel. {\tt vadim.indelman@technion.ac.il}. This work was  partially supported by the Israel Science Foundation (ISF), US NSF/US-Israel BSF,
	 and the Israeli Smart Transportation Research Center (ISTRC).}%
}

\begin{document}

\maketitle
\thispagestyle{empty}
\pagestyle{empty}

%%%%%%%%%%%%%%%%%%%%%%%%%%%%%%%%%%%%%%%%%%%%%%%%%%%%%%%%%%%%%%%%%%%%%%%%%%%%%%%%
% abstract
\input{00-abstract}

%%%%%%%%%%%%%%%%%%%%%%%%%%%%%%%%%%%%%%%%%%%%%%%%%%%%%%%%%%%%%%%%%%%%%%%%%%%%%%%%
% Introduction
\input{01-introduction}

%%%%%%%%%%%%%%%%%%%%%%%%%%%%%%%%%%%%%%%%%%%%%%%%%%%%%%%%%%%%%%%%%%%%%%%%%%%%%%%%
% Related Work
\input{02-relatedwork}

%%%%%%%%%%%%%%%%%%%%%%%%%%%%%%%%%%%%%%%%%%%%%%%%%%%%%%%%%%%%%%%%%%%%%%%%%%%%%%%%
% noatations
\input{03-formulation}

%%%%%%%%%%%%%%%%%%%%%%%%%%%%%%%%%%%%%%%%%%%%%%%%%%%%%%%%%%%%%%%%%%%%%%%%%%%%%%%%
% approach
\input{04-approach}
\input{05-results}

%%%%%%%%%%%%%%%%%%%%%%%%%%%%%%%%%%%%%%%%%%%%%%%%%%%%%%%%%%%%%%%%%%%%%%%%%%%%%%%%
% conclusions
\input{06-conclusions}

%%%%%%%%%%%%%%%%%%%%%%%%%%%%%%%%%%%%%%%%%%%%%%%%%%%%%%%%%%%%%%%%%%%%%%%%%%%%%%%%
% acknowledgment
\input{08-acknowledgment}
%%%%%%%%%%%%%%%%%%%%%%%%%%%%%%%%%%%%%%%%%%%%%%%%%%%%%%%%%%%%%%%%%%%%%%%%%%%%%%%%
% appendix
%\input{07-appendix}

\addtolength{\textheight}{-12cm}   % This command serves to balance the column lengths

%%%%%%%%%%%%%%%%%%%%%%%%%%%%%%%%%%%%%%%%%%%%%%%%%%%%%%%%%%%%%%%%%%%%%%%%%%%%%%%%

\bibliographystyle{plain}
\bibliography{/Users/tomyotam/research/references/refs.bib}
%\bibliography{/Users/indelman/Vadim/PROFESSIONAL/RESEARCH/PAPERS/References/refs.bib}

\end{document}

%% file: packages.tex
\usepackage{cite}
\usepackage{amsmath,amssymb,amsfonts}

\usepackage{amsthm}
\usepackage{algorithmic}
\usepackage{graphicx}

\usepackage{textcomp}
\usepackage{xcolor}
\usepackage{caption}
\usepackage{subcaption}
\usepackage[hidelinks]{hyperref} 

\makeatletter
\def\maketag@@@#1{\hbox{\m@th\normalfont\normalsize#1}}
\makeatother

\usepackage[linesnumbered,ruled,vlined]{algorithm2e}

%% file: commands.tex
\newtheorem{theorem}{Theorem}
\newtheorem{lemma}{Lemma}

\newtheorem{corollary}{Corollary}

\newcommand{\expt}[2]{\ensuremath{\underset{{#1}}{\mathbb E}{[{#2}]}}}

\newcommand{\prob}[1]{\ensuremath{\mathbb{P}({#1})}}
\newcommand{\entrp}[1]{\ensuremath{\mathcal{H}({#1})}}

\newcommand{\prtn}[4]{\ensuremath{{Z}^{#1_{#2}|#3_{#4}}}}

%% file: 00-abstract.tex
\begin{abstract}
Decision making under uncertainty is at the heart of any autonomous system acting with imperfect information. The cost of solving the decision making problem is exponential in the action and observation spaces, thus rendering it unfeasible for many online systems. This paper introduces a novel approach to efficient decision-making, by partitioning the high-dimensional observation space. Using the partitioned observation space, we formulate analytical bounds on the expected information-theoretic reward, for general belief distributions. These bounds are then used to plan efficiently while keeping performance guarantees. We show that the bounds are adaptive, computationally efficient, and that they converge to the original solution. We extend the partitioning paradigm and present a hierarchy of partitioned spaces that allows greater efficiency in planning. We then propose a specific variant of these bounds for Gaussian beliefs and show a theoretical performance improvement of at least a factor of 4. Finally, we compare our novel method to other state-of-the-art algorithms in active SLAM scenarios, in simulation and in real experiments. In both cases we show a significant speed-up in planning with performance guarantees.
\end{abstract}

%% file: 01-introduction.tex
\section{INTRODUCTION}
Autonomous agents must operate with imperfect information about their environments, dynamics, and measurements. \emph{Belief Space Planning} (BSP) is one of the fundamental problems one must solve for these autonomous agents to interact with the environment successfully. The modeling of the problem is such that we maintain a probability density function over the true state of the agent, which is unknown. We then reason about the evolution of this distribution in the future for different actions and possible observations. BSP uses two main models to evolve the distribution over the state:  motion and measurement models. The two models often have very distinctive and different effects. The motion model introduces noise by state uncertainty. The measurement model on the other hand, although noisy by itself, usually provides valuable information about the agent's pose and the map of the environment, which in turn reduces state uncertainty. 

One advantage of using the BSP formulation is the capacity to incorporate belief dependent reward functions, particularly, information-theoretic rewards, which is important for tasks such as; search and rescue, informative path planning and active classification. The ability to measure belief uncertainty and reduce it is key in solving such tasks, however, it comes with added computational complexity, especially when the observation space is high-dimensional. In a typical active \emph{Simultaneous Localization And Mapping} (SLAM) setting, future observations may include hundreds or even thousands of landmarks. Moreover, in a visual-based POMDP setting, even a single observation can be high-dimensional when the measurement model uses raw image inputs.

Solving the corresponding POMDP problem involves reasoning about different actions or policies, and for each, account for different possible observations. This leads to an exponential growth of the posterior beliefs, which in turn makes the planning problem NP hard\cite{Papadimitriou87math}. 

A possible approach to addressing this issue is to simplify the planning problem. One specific simplification method of interest, is forming analytical bounds on the expected reward and using the bounds to decide on the optimal action. Given that the bounds are easier to calculate than the expected reward, planning becomes more efficient.  

\begin{figure} [!t]%[!h]
	\hspace*{-4pt}
	\includegraphics[scale=0.4]{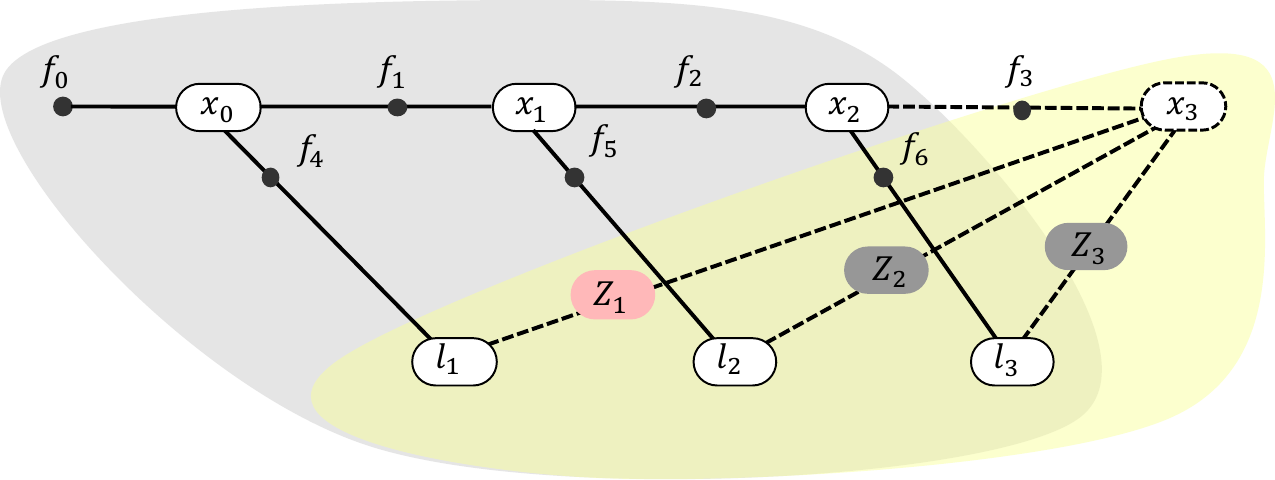}
	\scriptsize
	\caption{\scriptsize A prior factor graph is shown in the gray blob;  considering an action that leads to the factor $f_3$, we see the posterior factor graph in the yellow blob. The posterior graph includes the future +random measurements $Z_1,Z_2,Z_3$. Different sets of random measurements are assigned different colors which represents one possible partitioning. These sets are used to bound the conditional Entropy of the entire posterior graph.}
	\normalsize
	\label{fig: factor graph partitioning}
	\vspace{-10pt}
\end{figure}

In this work, we present a novel approach of simplification of the POMDP planning problem, specifically to the \emph{multivariate observation space}. When the observation space is high-dimensional, the calculation of expected information-theoretic rewards becomes expensive. We show that by partitioning the random variables that model future possible measurements into sets, this calculation becomes more efficient by developing novel bounds on the expected reward.

To illustrate, consider a factor graph representation for a given belief at planning time $t=2$, as in Fig~\ref{fig: factor graph partitioning}. In this toy example, we sketch out how in planning, a specific action leads to 3 new random observations involving different landmarks. One can think of a simplified setting, that takes into account only a subset of these random observations. This subset of observations is a partition of the multivariate random variable representing the 3 original observations. This multivariate random variable defines in this case the observation space for this specific time step, such that, its partitioning is a partition of the observation space. Similarly, in visual based POMDP setting, a partition of the observation space may correspond to a partition of the random variables representing future image pixels into subsets.

Specifically, our main contributions are as follows.  
We introduce the novel concept of observation space partitioning. Using the partitioned space, we show the relation to the original problem by deriving analytical bounds on the expected entropy that hold for all families of belief distributions. We present a partition tree that allows greater efficiency as we go down its hierarchy. We show that these bounds are adaptive, computationally cheaper, and that they converge to the original solution. Moreover, we show one possible realization of this general framework for an active SLAM scenario involving multivariate Gaussian distributions and present a hierarchy of efficient implementations. The speed up ranges from a factor of 4 for the least efficient one, and a speed-up from a cube to linear time for the most efficient one. We demonstrate the use of this framework in both simulated and real experiments with significant speed-up.

%% file: 02-relatedwork.tex
\section{RELATED WORK} \label{RelatedWork}
POMDP has been widely used as a model for decision-making under uncertainty, despite the fact  that obtaining the optimal solution for the planning problem is known to be intractable \cite{Papadimitriou87math} \cite{Kaelbling98ai}.
While the standard POMDP is formulated with state dependent reward functions, it is possible to extend this framework to include belief dependent reward as well; examples for such frameworks are $\rho$-POMDP proposed by \cite{Araya10nips} and BSP by \cite{Platt10rss} and \cite{VanDenBerg12ijrr}. This extension is essential for tasks such as terrain monitoring \cite{Popovic17icra}, information gathering \cite{Stachniss05rss}, and active SLAM \cite{Kim14ijrr}.
 Various approximation methods to solving POMDPs were proposed, however even approximating a solution is challenging since most real-world problems incorporate continuous spaces. A sparse tree representation combined with a Monte Carlo Tree Search was explored in \cite{Somani13nips}, in order to approximate near-optimal policies, however, it is not suitable for information-theoretic rewards.  A particle-based approach to represent the belief was taken by \cite{Sunberg18icaps}, proposing two different algorithms, one of them catering specifically to belief-dependent rewards. An abstraction of the observation model was studied in \cite{Barenboim22ijcai}, for a low-dimensional state-space and non-parametric beliefs, allowing speed-up compared with other approximation methods. Similarly to our method, it is using the observations as a means for simplification, where the main difference is that the mechanism for simplification is clustering of observation samples (post-hoc), while our method is aimed at the observation space directly, by partitioning the high-dimensional observation space and reducing its dimensionality.  
 
Early works have identified the importance of quantifying the information contained in observations. The quantitative value of possible measurements is presented in \cite{Davison05iccv}, in an effort to incorporate this value measure in inference. While it did allow for better selection of measurements using heuristics, it did not provide any optimality guarantees. Using the sub-modularity of mutual information, \cite{Krause08jmlr} formulated a near-optimal algorithm for sensor placement, which can alternatively cast as a myopic planning algorithm, however it was limited to Gaussian processes.

Probabilistic graphical models have seen substantial traction in the world of inference, one of the most popular one was introduced by \cite{Kaess08tro} and was later extended at \cite{Kaess12ijrr}. The latter utilized the structure of SLAM problems to be encoded in a Bayes-tree, which allowed for an incremental update of the posterior belief with incoming information.  
The computational complexity for big SLAM problems is still very high, and there have been many works that have tried to reduce the computational complexity of inference, considering probabilistic graphical models. Some of the more popular SLAM methods using graphical models were introduced by   \cite{Khosoussi20SPAR} presented a graph-theoretic approach to the problem of designing sparse reliable pose-graph SLAM in the context of measurement selection, both \cite{CarlevarisBianco14tro} and\cite{Kretzschmar12ijrr} showed methods for compressing a factor graph, and \cite{Zhang15cvpr} reviewed how feature selection based on some defined scores can improve localization and data association, proposing a greedy algorithm that relies sub-modularity as well.

Other works have studied simplification method for planning problems; \cite{Smith04uai} presented a heuristics-based bounds on the value function, to guide local updates; a belief compression method was proposed in \cite{Roy05jair}, but it lacked guarantees on planning performance. Several works have put forward simplified methods while providing guarantees; for a Gaussian high-dimensional state, \cite{Indelman16ral} proposed a transformation of the original information space to a conservative one, by decoupling all state variables. More general approaches were studied in \cite{Elimelech22ijrr} and \cite{Zhitnikov22ai}; the former outlined a theoretical framework for simplification in general while demonstrating said framework for a sparse approximation of the initial belief, while the latter studied a simplification in risk averse planning, while considering a distributional perspective. 
%Reducing the number of samples representing the posterior belief distribution was used as a simplification in \cite{Sztyglic22iros}, but similarly to many non-parametric methods, it was aimed at a low-dimensional state-space setting.  That work was recently extended \cite{Zhitnikov24ijrr_accepted} to an adaptive multi-level simplification framework considering belief-dependent rewards. 
A multi-level adaptive simplification framework considering belief-dependent rewards was developed in \cite{Sztyglic22iros, Zhitnikov24ijrr_accepted}, based on novel bounds \cite{Sztyglic22iros} on a differential  entropy estimator \cite{Boers10fusion} that utilize a reduced number of state-samples. However, similarly to many non-parametric methods, these works  aimed at a low-dimensional state-space setting. 
%Reducing the number of samples representing the posterior belief distribution was used as a simplification in \cite{Sztyglic22iros}, but similarly to many non-parametric methods, it was aimed at a low-dimensional state-space setting. That work was recently extended \cite{Zhitnikov24ijrr_accepted} to an adaptive multi-level simplification framework considering belief-dependent rewards.  
These works and other simplification methods are complementary to ours and can be used alongside one another, as none of these works considered simplification to the observation space itself.

%% file: 03-formulation.tex
\section{NOTATIONS AND PRELIMINARIES} \label{notations section}
\subsection{$\rho$-POMDP}
A discrete-time POMDP models an agent decision process by outlining the dynamics of the interaction between the agent and its environment. It is defined as the tuple $(\mathcal{X},\mathcal{A},\mathcal{Z},T,O,\rho)$, consisting of a state, action and observation spaces, a transition and observation models, and a reward function. We assume a Markovian transition model, i.e. $T(X,a,X')=\prob{X'\vert X,a}$, and that each measurement is conditionally independent given the state, i.e. $O(X,z)=\prob{z\vert X}$. 

Since the agent only observes the environment through noisy measurements, it must maintain a probability distribution over the true state; we denote this distribution as a \emph{belief}. The belief and the propagated belief, are defined as, respectively:
 \begin{align}
 	&b_k \triangleq  b[X_k] = \mathbb P (X_k\vert z_{0:k},a_{0:k-1})\triangleq\mathbb P (X_k\vert h_k) \\
 	&b^-_k \triangleq b[X_k^-] = \mathbb{P} (X_k\vert z_{0:k-1},a_{0:k-1})\triangleq\mathbb P (X_k\vert h^-_k).
 \end{align}
 At each discrete time step this belief is updated with new motion and observation information according to Bayes' rule. Given an action $a_k$ and observation $z_{k+1}$ the  belief is updated according to:
 \begin{equation*}
 	b_{k+1}=\eta \int_{X_k} \prob{X_{k+1}\vert X_k,a_k}\prob{z_{k+1}\vert X_{k+1}}b_k dX_{k},
 \end{equation*}   
 where $\eta$ is a normalization constant.
A policy $\pi: \mathcal{B} \mapsto \mathcal{A}$ maps belief states to actions. Usually, only state dependent rewards are considered in the POMDP setting. BSP and later on
$\rho$-POMDP, extend the POMDP model to include belief dependent rewards. For some finite planning horizon $\ell$, the \emph{value} of a policy $\pi$, is defined as the expected cumulative reward received by following $\pi$ with initial belief $b_k$: 
\begin{equation}
	V^\pi(b_k) = \rho(b_k,\pi_k(b_k))+ \!\!\! \mathop{\mathbb{E}}_{Z_{k+1: k+\ell}} \left [\sum_{i=k+1}^{k+\ell}\rho(b_i,\pi_i(b_i))\right].
\end{equation} 
Solving a POMDP is equivalent to finding the optimal policy $\pi ^*$ such that the value function is maximized.  

A general form of the reward function $\rho(b_i,\pi_i(b_i))$ can be expressed as a sum of a belief-dependent component $R(b_i)$, and a state-dependent component $R^X(b_i,\pi_i(b_i))$,
 \begin{equation} \label{eq: reward factorization}
 	\rho(b_i,\pi_k(b_i)) = R(b_i) + \alpha \cdot R^X(b_i,\pi(b_i)),
 \end{equation}
 where %$R(b_i)$ and $R^X(b_i,\pi(b_i))$ are the belief-dependent, and a state-dependent components, respectively, and
  $\alpha$ is some weight. The belief-dependent component could correspond to an information-theoretic reward, such as differential entropy and information gain. The state-dependent component can be expressed as $R^X(b_i,\pi(b_i))=\expt{X_i \sim b_i}{r(X_i,\pi_i(b_i))}$. For instance, $r(X_i,\pi_i(b_i))$ could represent distance to goal or obstacle, and  control effort.
  
In most cases, $R(b_i)$ is the computationally expensive component, and as such, it is the focal point of this work. %All future references of the reward function pertain to this component. 
We consider a particular instance of this  reward function, namely, differential Entropy: 
\begin{equation}
	R(b)\triangleq -\entrp{X }\equiv \mathop{\mathbb{E}}_{X \sim b}\left(\log b[X]\right),
\end{equation}
where $X$ is a random variable distributed according to $b[X]$.

If both $X,Z$, are treated as random variables, the expected reward becomes the conditional entropy of these random variables, i.e.
\begin{equation} \label{eq:ExpEntropyGeneral}
	\expt{Z}{R(b)}= - \entrp{X \mid Z} = -\expt{Z}{\entrp{X \mid Z=z}}. 
\end{equation}
Thus, the expected reward at each $i$th look ahead step
can be equivalently written as:
\begin{equation}\label{eq:ExpEntropy}
		\expt{Z_{k+1:i}}{R(b_{i},\pi(b_i))}\!=\!
		- \mathcal{H}(X_{i}\vert Z_{k+1:i}),
\end{equation}
where the future observations are drawn from the distribution $\prob{ Z_{k+1:i} \mid b_k,  \pi}$ and $i\in[k+1,k+\ell]$.

In this paper we consider an open loop setting, as formulated in the next section.
 \subsection{Active SLAM} \label{active slam}

 Let $x_k$ be the state of the agent at time $k$, and $X_k$ be the joint state of the agent's trajectory and environment, e.g.~landmarks, up to, and including time $k$. We define $z_k\triangleq\{z_k^0,...,z_k^m\}$ as the set of all measurements observed at time $k$, and $z_{0:k}\triangleq\{z_0,...,z_k\}$ as the set of all measurements until time $k$. Similarly we define $a_{0:k}\triangleq\{a_0,...,a_k\}$ as the set of all actions  until time $k$.
 Assuming static landmarks, the motion of the agent, and the observations it receives are modeled as: 
 \begin{equation} \label{models}
 	x_{k+1}=f(x_k,a_k)+w_k ,\quad\quad z_k=h(X_k)+v_k,
 \end{equation} 
where $f$ and $h$ are some deterministic functions, and $w_k$ and $v_k$ are their process noise, respectively. 

We denote the Data Association vector at time $k$ as $\beta_k$. The dimensionality of $\beta_k$, is equal to $X_k$ excluding $x_k$, and is composed of binary entries, where each entry indicates whether the corresponding state was involved in a measurement at that given time step. We assume that each measurement involves the current pose $x_i$, such that $\beta$ does not account for it.  For example, at time step $k=3$, for a prior state vector of dimensionality $5$ and a measurement involving the third and fifth components of the state (e.g.~observation of two landmarks): $\beta_3 = \begin{pmatrix} 0 & 0 & 1 & 0 & 1\end{pmatrix}^T$.

Alternatively, we can express the objective function using the data association vector. For a given action sequence $a_{k:k+\ell}$, the objective function is defined as:
\begin{equation} \label{objective conditioned beta}
	J(b_k,a_{k:k+\ell})\triangleq \mathop \mathbb{E}_{\tilde \beta}\left[\mathop\mathbb{E}_{\tilde Z \vert \tilde \beta}\left[\sum_{i=k+1}^{k+\ell} R(b_{i},a_i)\right]\right], 
\end{equation}
where $\tilde \beta \triangleq\beta_{k+1:k+\ell}$ and $\tilde Z \triangleq Z_{k+1:k+\ell}$. In this scenario, $\tilde \beta$ dictates the number of measurements and the states involved, while $\tilde Z$ encodes the information about the distribution of those measurements.
Combining Bayes rule and the properties of the models, we can factorize a posterior belief $b[X_{k+\ell}] $ given $\tilde \beta$, into prior belief, motion and measurement factors:
\begin{equation} \label{eq: BeliefModelFactorization}
	b[X_{k+\ell}]  \propto b[X_{k}] \! \mathop  \prod_{i=k+1}^{k+\ell} \! \prob{x_i\vert x_{i-1},a_{i-1}} \prob{z_i \mid X_{i}, \beta_i},
\end{equation} 
where
\begin{equation}
	\prob{z_i \mid X_{i}, \beta_i} = \prod_{j=1}^{m_i(\beta_i)} \prob{z_{i,j}\vert x_i,  X_i^{\beta_i(j)}}, 
\end{equation}
where $m_i$ is the number of measurements at the $i$th time step, and $X_i^{\beta_i(j)}$ represents the involved state in the $j$th  measurement, at the $i$th time step, both as a function of $\beta_i$.
We denote the set of all involved state variables for a given time step as:
\begin{equation} \label{eq: InvolvedState}
	X_k^\text{inv}=\{X_k^{\beta_k(j)} | j\in \mathcal J \}, 
\end{equation}
where $ \mathcal J = \{1,2,..,m_k(\beta_k)\}$. 

When the models in \eqref{models} are linear, with zero-mean Gaussian noise, i.e. $w_k\sim \mathcal{N}(0, W_k)$ and $v_k\sim \mathcal{N}(0, V_k)$,  and the prior belief is Gaussian, it can be shown that the posterior belief is also Gaussian. In such case, the Entropy of the posterior belief can be expressed as:
\begin{equation} \label{GaussianEntropy}
	\entrp{X}=\frac{1}{2}(\ln\vert\Sigma \vert+N\ln(2\pi e)),
\end{equation}
where $b[X] = \mathcal{N}(\mu, \Sigma)$, with mean $\mu \in \mathbb{R}^N$ and covariance matrix $\Sigma \in \mathbb{R}^{N \times N}$. The above also holds when the belief is modeled or approximated as Gaussian for nonlinear models. 
The inverse of the covariance is known as the information matrix, such that $\Sigma^{-1}_k=\Lambda_k$. At each given time step, the information matrix of a posterior belief $b_{k+1}$, can be decomposed to: 
\begin{equation}
	\Lambda_{k+1} = \Lambda_k^{\text{Aug}} + F^T W_{k+1}^{-1} F + H^T V_{k+1}^{-1} H,
\end{equation}
where  $\Lambda_k^{\text{Aug}}$ is the prior information matrix of $b_k$, augmented with zeros to accommodate new states, $W_{k+1}$ and $V_{k+1}$ are the noise covariance matrices of the motion and measurement models respectively, and $F\triangleq \nabla f$ and $H\triangleq \nabla h$ are the Jacobians of the motion and measurement functions respectively. The stacked matrices of $F^T W_{k+1}^{-1} F$ and $H^T V_{k+1}^{-1} H$ are denoted as the collective Jacobian $\tilde A_k$, see \cite{Indelman15ijrr} for details. If we combine the collective Jacobians of consecutive time steps, i.e. $\tilde{A}_{k+1},\tilde{A}_{k+2},\hdots,\tilde{A}_{i}$ we get the following update rule: 
\begin{equation} \label{InfoMatUpdate}
	\Lambda_{i} = \Lambda_k^{\text{Aug}}+\tilde A^T_{k+1:i} \cdot \tilde A_{k+1:i},
\end{equation}
where $\tilde A_{k+1:i}$ is the collective Jacobian of the motion and measurement factors of \eqref{eq: BeliefModelFactorization}, from the time step $k+1$ until $i$.   

For the sake of readability, we drop the notation of the history $h_i$ from now on, but assume all distributions of a given time step are conditioned on the history available at the beginning of the planning session. We denote the collective Jacobian of a given horizon $\tilde A_{k+1:i}$ simply as $A_{i}$.

%% file: 04-approach.tex
\section{PROBLEM FORMULATION AND APPROACH}

In the following section we show how to use observation space partitioning to simplify the BSP problem. In order to choose the optimal action from a pool of candidate actions, one needs to evaluate the expected reward function $\mathbb{E}_{Z_{k:i}} \left (\rho(b_i, a_i)  \right)$ at each future time instant $i$ in the planning horizon $\ell$ for each action sequence $a_{k:k+\ell}$. Instead, one can evaluate bounds on the expected reward function as a proxy,  
	   \begin{align}\label{eq:bounds_rhoz}
		\mathcal{LB}^{\rho}_{i} \leq \mathop \mathbb{E}_{Z_{k:i}} \left (\rho(b_i, a_i)  \right) \leq \mathcal{UB}^{\rho}_{i}.
	\end{align} 
In this work we focus on bounding the information-theoretic reward $R(b_{i})$, since it is typically the computational bottleneck with respect to the state-dependent reward,
\begin{equation} \label{eq: GeneralBounds}
	\mathcal{LB}_{i}\leq \mathop \mathbb{E}_{Z_{k:i}} \left (R(b_{i}) \right)\leq \mathcal{UB}_{i},
\end{equation}
and thus
\begin{align}\label{eq:bounds_ho2}
	\mathcal{UB}^{\rho}_{i} &\triangleq 
	\mathcal{UB}_i + \alpha \mathbb{E}_{Z_{k:i}}  [R^X(b_i,a_i)] ,\\
	\mathcal{LB}^{\rho}_{i} &\triangleq 
	\mathcal{LB}_i + \alpha \mathbb{E}_{Z_{k:i}} [R^X(b_i,a_i)].
	\label{eq:bounds_ho2b}
\end{align}
Note that, considering an open-loop setting, the expected state-dependent  reward $\mathbb{E}_{Z_{k:i}} [R^X(b_i,a_i)]$ component in $\rho(b_i,a_i) $ does not depend on future observations, which marginalize out:   since ~$R^X(b_i,a_i)= \expt{X_i \sim b_i}{r(X_i,a_i)}$, $\mathbb{E}_{Z_{k:i}} [R^X(b_i,a_i)]$ can be expressed as
%expectation over future observations in the state-dependent  reward cancels, i.e.~$\mathbb{E}_{Z_{k:i}} [R^X(b_i,a_i)]dZ_{k:i}$ can be expressed as 
\begin{align*}
	%&\mathbb{E}_{Z_{k:i}} [R^X(b_i,a_i)]dZ_{k:i}=
	&\int_{Z_{k:i}} \prob{Z_{k:i} \mid b_k, a_{k:i}} R^X(b_i,a_i) dZ_{k:i}
	\\	
	&=\!\! \int_{X_i} \prob{X_i \mid b_k, a_{k:i}} r(X_i,a_i) dX_i
	\\
	&=\mathbb{E}_{X_k \sim b_k} \mathbb{E}_{X_{k+1}\sim \prob{.\mid X_k,a_k}} \cdots \mathbb{E}_{X_{i}\sim \prob{.\mid X_{i-1},a_{i-1}}} [r(X_i,a_i)].
\end{align*}
Therefore, the bounds from \eqref{eq:bounds_ho2} assume the form
\begin{align}
	\mathcal{UB}^{\rho}_{i} &\triangleq 
	\mathcal{UB}_i + \alpha \mathbb{E}_{X_{i}\sim \prob{.\mid b_k, a_{k:i-1}}}  [r(X_i,a_i)] ,\\
	\mathcal{LB}^{\rho}_{i} &\triangleq 
	\mathcal{LB}_i + \alpha \mathbb{E}_{X_{i}\sim \prob{.\mid b_k, a_{k:i-1}}}  [r(X_i,a_i)],
\end{align}
which means that the considered simplification of the observation space does not have an impact on the state-dependent reward component in an open-loop setting, which in any case does not depend on future observations. 

As stated, this result is only valid for the open-loop setting considered herein, while in a close-loop setting we would remain with the bounds \eqref{eq:bounds_ho2}-\eqref{eq:bounds_ho2b} (with appropriate replacement of actions to policies; see also \cite{Barenboim22ijcai}).

In the same manner we can bound the objective function by summing up the bounds over the expected reward function for each of the time steps, 
\begin{equation}
	\sum_{i=k+1}^{k+\ell} \mathcal{LB}_{i}^{\rho}\leq J\left( b_k,a_{k:k+\ell-1}\right)\leq \sum_{i=k+1}^{k+\ell}\mathcal{UB}_{i}^{\rho}.
\end{equation}

\setlength{\belowcaptionskip}{-12pt}
\begin{figure} [!h]
	\hspace*{+25pt}
	\includegraphics[scale=0.6]{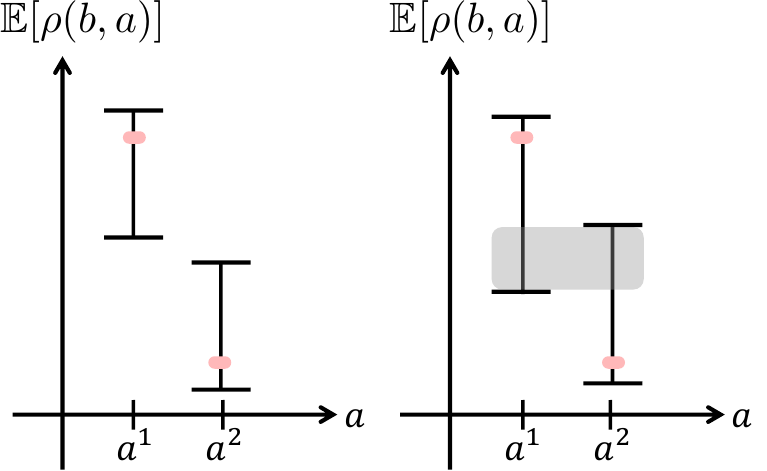}
	\scriptsize
	\caption{\scriptsize The expected reward is shown in red within its bounds. On the left we can select the optimal action based on the bounds alone, on the right the worst-case loss is shaded in gray.} 
	\normalsize
	\label{fig:BoundsConcept}
	\vspace{6pt}
\end{figure}
\setlength{\belowcaptionskip}{0pt} 

Under the assumption that these bounds can be efficiently calculated, it is easier to select actions based on the reward bounds; see illustration in Fig.~\ref{fig:BoundsConcept}. 
We can think about two distinct actions, $a^1$ and $a^2$. For each action we calculate the expected reward bounds and face two cases: in the first, the bounds do no overlap and we can select the optimal action, in the second, the bounds do overlap and we choose between tightening the bounds such that they do not overlap, or selecting an action while bounding the loss. Previous works have developed such bounds considering various simplification methods, as discussed in Section  \ref{RelatedWork}. 

In this work, we put forward a fundamental simplification that applies to the observation space itself. Specifically, we propose a partitioning of the observation space and develop expected reward bounds that are a function of this partitioning.

\subsection{Partitioning of a Multivariate Observation Space }\label{sec:Partitioning}
Consider a multivariate random variable $Z$ that represents future observations and the corresponding observation space $\mathcal{Z}$. In general, $Z$ can be represented by 
\begin{equation} \label{MasurementVector}
	Z =(Z^1, Z^2,\hdots, Z^m),
\end{equation}
where $Z^i$ is a random variable defined by a given measurement model, and $m$ is the number of such random variables.
We can now partition $Z$ into different subset of components, for example, consider the partitioning  $Z^s \in \mathcal{Z}^s$ and $Z^{\bar s} \in \mathcal{Z}^{\bar s}$,  such that: 
\begin{equation} \label{MasurementPartitioning}
\begin{split}
	&Z^s =\{Z^1, Z^2,\hdots, Z^n\}, \\
	&Z^{\bar{s}} =\{Z^{n+1}, Z^{n+2},\hdots, Z^m\},
\end{split}
\end{equation}
where $Z=Z^s\cup Z^{\bar s}$, and their corresponding subspaces $\mathcal{Z}=\mathcal{Z}^s + \mathcal{Z}^{\bar s}$. 

In the subsequent section, we derive bounds on the expected reward that are a function of this partitioning, such that $\mathcal{LB}$ and $\mathcal{UB}$ from \eqref{eq: GeneralBounds} become:
\begin{align}
	&\mathcal{LB}_{i}(b_{i},{Z^s_{k:i}},Z^{\bar{s}}_{k:i}),\\ 
	&\mathcal{UB}_{i}(b_{i},{Z^s_{k:i}}),
\end{align} 
where ${Z^s_{k:i}},Z^{\bar{s}}_{k:i}$ are subsets of measurements from time step $k$ to $i$.
The partitioning can be applied (but not limited) to two different observation spaces \textemdash  a raw measurement such as pixels in an image, or to a SLAM scenario where the measurement model is defined by \eqref{models}, and $\beta$ dictates the dimension of the measurements, e.g.~number of observed landmarks considering a future camera pose. 

To illustrate the computational advantage of these bounds, we apply partitioning to a raw image measurement of size $20 \times 20$ binary pixels. Each pixel is represented by a random variable $Z^{x,y}\in \{0,1\}$, where $x,y$ denote the pixel location on the sensor, and $Z\in \mathcal{Z}\subseteq (\mathbb{F}_2)^{400}$. We must consider all of the different permutations for each of those pixels, $2^{400}$ in total, which defines $\vert \mathcal{Z} \vert$ in this case. For example, if we partition $Z^s$ to represent the left half of the image, and $Z^{\bar{s}}$ to represent the right half, we need only to consider $2^{200}$ permutations for $Z^s$, and another $2^{200}$ for $Z^{\bar{s}}$, $2^{201}$ in total.  

\paragraph*{Hierarchical Partitioning}
there can be higher levels of partitioning, breaking down further a given measurement set into two sets. To encode this partitioning scheme we index the partitioning depth, and the number of nodes at a given depth. Each partitioned set is given a unique encoding denoted as $\prtn{n}{i}{m}{j}$, where $n$ is the node number at the $i$th partitioning level, and $m$ is the node number at the parent partitioning level $j$. We note this slight abuse of notation in regard to \eqref{MasurementPartitioning}, but it should be clear from the context, which notation is used. If two sets share a parent set, we consider them a base subset and its complement. For example, $\prtn{4}{3}{1}{2}$ represents the 4th node of level 3, where the parent set is the first node of level 2. In this new notation, $Z^s,Z^{\bar s} $ becomes $\prtn{1}{1}{1}{0},\prtn{2}{1}{1}{0}$ or equivalently $\prtn{1}{1}{1}{0},\prtn{\bar 1}{1}{1}{0}$. Overall, for $Z \in \mathbb R^m$, it is possible to create a partition hierarchy of depth $\log_2m$, as illustrated in Fig.~\ref{fig: PartitionTree}. 

\begin{figure} [!h] 
	\includegraphics[width=\columnwidth]{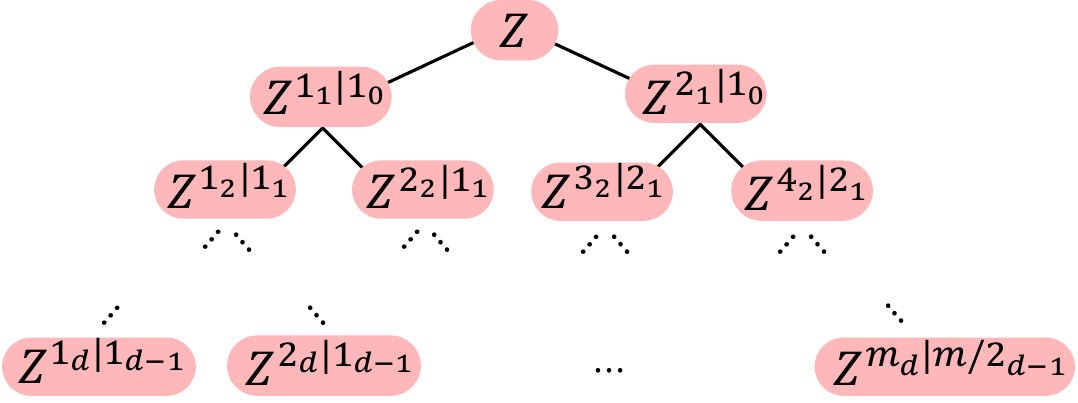}
	\scriptsize
	\caption{\scriptsize An Illustration of a possible partition tree. At each level of partitioning, we split a measurement set into two. For $Z\in \mathbb R^m$, the depth of the tree is $d=\log_2m$.}
	\normalsize
	\label{fig: PartitionTree}
\end{figure}

\subsection{Bounds on Expected Entropy}

In this section we use Measurement Partitioning to derive novel information theoretic reward bounds considering arbitrary distributions, we begin by introducing a helpful lemma.

\begin{lemma} \label{EntropyBayes}
	The conditional Entropy can be factorized as: 
\end{lemma}

\begin{equation} \label{reward function factorization}
		\mathcal{H} \left(X | Z \right) = 
		\mathcal{H} \left(Z |X\right) +  \mathcal{H} \left(X\right)-\mathcal{H}\left(Z \right).
\end{equation}  

\begin{proof}
	For two random variables that have a joint Entropy $\mathcal{H}(X,Z)$, we know that conditioning on $Z$ yields $\mathcal{H}(X,Z)=\mathcal{H}(X\vert	Z)+\mathcal{H}(Z)$. Similarly, conditioning on $X$ yields $\mathcal{H}(X,Z)=\mathcal{H}(Z\vert	X)+\mathcal{H}(X)$. Combining both equations and rearranging terms we obtain the desired equality.
\end{proof}

There are three quantities we need to evaluate in \eqref{reward function factorization}; the Entropy of the likelihood of a measurement, which is done via the measurement model, the prior Entropy, and the Entropy of the measurement given the history. The prior Entropy is common to all actions and can be calculated once.  The rest of the terms involve future random measurements and we shall apply partitioning to these terms.

\begin{lemma}\label{lem:Factorization}
	Given two sets of expected measurements and the partitioning $Z=Z^s\cup Z^{\bar s}$ from \eqref{MasurementPartitioning}, the conditional Entropy can be factorized as:
\begin{equation} \label{reward factorization with measurements}
			\entrp{X | Z} \!\! =  \!\!\entrp{Z^{s}\vert X}+\entrp{Z^{\bar{s}}\vert X}  - \entrp{Z^{s},Z^{\bar{s}}}  +\entrp{X}.
\end{equation}
\end{lemma}
\begin{proof}
	Using \eqref{reward function factorization}, we can rewrite the conditional Entropy as $\mathcal{H} \left(X | Z \right) =  \entrp{Z^{s},Z^{\bar{s}}\vert X}
		-\entrp{Z^{s},Z^{\bar{s}}}+\entrp{X}$. The measurements are independent given the state, such that $\mathbb P \left(Z^{s},Z^{\bar{s}}\vert X\right)=\mathbb P \left(Z^{s}\vert X\right)\mathbb P \left(Z^{\bar{s}}\vert X\right)$. The Entropy of two independent random variables is just the sum of individual Entropies such that $\mathcal{H} \left(Z^{s},Z^{\bar{s}}\vert X\right)=\mathcal{H} \left(Z^{s}\vert X\right)+\mathcal{H} \left(Z^{\bar{s}}\vert X\right)$.
\end{proof}
Having established what measurement partitioning looks like in \eqref{MasurementPartitioning}, we can use it to bound the expected reward. The following two Theorems present our main result.
\begin{theorem} \label{th: UB}
	The conditional Entropy can be bounded from above by:
	\begin{equation} \label{eq: RewardUB}
	\entrp{X \vert Z} \leq \mathcal{UB}\triangleq \mathcal{H}\left(Z^{s}\vert X\right)+\mathcal{H}\left(X\right)-\mathcal{H}\left(Z^{s}\right).
	\end{equation}
\end{theorem}
\begin{proof}
	It is not difficult to show that $\mathcal{H}\left( X \vert Z^{s}\right) - \mathcal{H} \left(X | Z \right) = \mathcal {I} (X \vert Z^{s} ; Z \setminus Z^s)$. Recalling that the mutual information between two random variables is always non-negative,  we get:
	\begin{equation}
		\mathcal{H} \left(X | Z \right) \leq  \mathcal{H}\left( X \vert Z^{s}\right).
		\label{eq:ConditioningArg}
	\end{equation}
	We denote this as the conditioning argument, i.e. conditioning on a random variable always reduces Entropy, and refer to it later. 
 Using Lemma \ref{EntropyBayes} we get 
	$\mathcal{H}\left( X \vert Z^{s}\right) = \mathcal{H} \left(Z^{s} |X\right) +  \mathcal{H} \left(X\right)-\mathcal{H}\left(Z^{s} \right)$.
\end{proof}
\begin{theorem} \label{th: LB}
	The conditional Entropy can be bounded from bellow by:
	\begin{multline} \label{eq: RewardLB}
		\entrp{X \vert Z} \geq \mathcal{LB}\triangleq \entrp{Z^s \mid X}
	+\entrp{Z^{\bar{s}}\vert X} \\
	 -\entrp{Z^{s}} -\entrp{Z^{\bar{s}}}+\entrp{X}.
	\end{multline}
\end{theorem}
\begin{proof}	
	The joint Entropy of the measurements can be written as $\mathcal{H}\left(Z^{s},Z^{\bar{s}}\right)=\mathcal{H}\left(Z^{s}\vert Z^{\bar{s}}\right)+\mathcal{H}\left(Z^{\bar{s}}\right)$. From the conditioning argument in \eqref{eq:ConditioningArg}, we get $\mathcal{H}\left(Z^{s}\vert Z^{\bar{s}}\right) \leq \mathcal{H}\left(Z^{s}\right)$,   such that $\mathcal{H}\left(Z^{s},Z^{\bar{s}}\right)\leq\mathcal{H}\left(Z^{\bar{s}}\right)+\mathcal{H}\left(Z^{s}\right)$. Rearranging the last inequality concludes the proof.
\end{proof}
\normalsize

We can think of this lower bound from the perspective of mutual information: the difference between  the original quantity $\mathcal{H}\left(Z^{s},Z^{\bar{s}}\right)$ and the quantities $\mathcal{H}\left(Z^{s}\right)$ and  $\mathcal{H}\left(Z^{\bar{s}}\right)$ is exactly $\mathcal{I}\left(Z^{s};Z^{\bar{s}}\right)$, such that the lower bound double counts the mutual information between the measurement sets, see Fig.~\ref{fig: MutualInfoFig}.

\begin{figure} [!h] 
	\hspace*{+25pt}
	\includegraphics[scale=0.5]{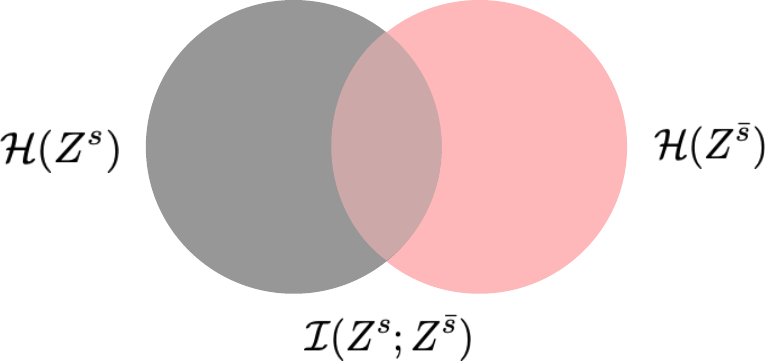}
	\scriptsize
	\caption{\scriptsize Visualizing the marginal Entropies of two measurement variables, the lower bound double counts the overlapping region which is the mutual information between those variables.}  
	\normalsize
	\label{fig: action selection with bounds}
	\vspace{6pt}
	\label{fig: MutualInfoFig}
\end{figure}

In some specific cases, it might be easier to work with posterior beliefs, i.e., with expressions of the form $\prob{X | \cdot }$. Such cases include situations where we have closed form expression to a Bayesian belief update, for example when the belief is modeled as a Gaussian.
In such cases it makes sense to rearrange the bounds as follows: 
\begin{corollary} The conditional Entropy can be bounded by:
	\begin{align}
	\begin{split} \label{eq:RewardLBParam}
	\mathcal{LB} = \entrp{X | Z^{s}} +\entrp{X | Z^{\bar{s}}}  - \entrp{X},
	\end{split}\\
	\begin{split} \label{eq:RewardUBParam}
	\mathcal{UB} = \entrp{X | Z^{s}}.
	\end{split}	
	\end{align}

\end{corollary}  
This can be obtained directly from \eqref{eq: RewardUB} and \eqref{eq: RewardLB} using Lemma \ref{EntropyBayes}.

We note that if one does not have access to the underlying belief distribution, one must resort to approximation of the expected reward, i.e.~conditional Entropy. This is usually done by sampling the closed-form expressions for the motion and measurement models. Such cases are outside of the scope of this work.

\subsection{Bounds with Hierarchical Partitioning}

In the previous section, we formed bounds based on a partitioning of measurement sets, by double counting the mutual information between said sets. This intuition applies to the hierarchical partitioning from Section \ref{sec:Partitioning} as well, allowing us to create a hierarchical notion of the bounds, starting with the lower bound. To formulate this idea we define a new operator: 
\begin{equation} \label{eq: g-Operator}
	g(  Z^{s} ,Z^{\bar{s}}) =\entrp{X |Z^{s}} + \entrp{X | Z^{\bar{s}}} - \entrp{X},
\end{equation}
where $g( Z ,\emptyset) =g(\emptyset, Z ) \triangleq \entrp{X | Z} -\entrp{X}$, and 
$g( \emptyset ,\emptyset) \triangleq -\entrp{X} $. 

Using this operator, \eqref{eq:RewardLBParam} can be expressed as $\mathcal{LB} = g( Z^{s} ,Z^{\bar{s}})$ and \eqref{eq:RewardUBParam} can be expressed as $\mathcal{UB} = g( Z^{s} ,\emptyset) +\entrp{X}$. Given that measurements sets can be hierarchically partitioned further we can formulate their bounds. 
\begin{theorem} \label{PartitionDepthLB} For the sets $Z^{s} ,Z^{\bar{s}}$, and their respective children $Z^{s_1},Z^{s_2} ,Z^{\bar{s_1}}, Z^{\bar{s_2}}$ the following holds:
\begin{align*}
     &g(Z^{s} ,Z^{\bar{s}})  \geq \\
     &g(Z^{s_1},Z^{s_2}) + g(\emptyset ,Z^{\bar{s}})  \geq \\
     &g(Z^{s_1},Z^{s_2}) + g(Z^{\bar{s_1}}, Z^{\bar{s_2}}) + g(\emptyset ,\emptyset) 
\end{align*}
\end{theorem}
\begin{proof}
From \eqref{eq:RewardLBParam} we know that $$\entrp{X | Z} \geq \entrp{X | Z^{s}} +\entrp{X | Z^{\bar{s}}}  - \entrp{X},$$ 
substituting $\entrp{X | Z^s}$ instead of $\entrp{X | Z}$ yields $$\entrp{X | Z^s} \geq \entrp{X | Z^{s_1}} +\entrp{X | Z^{{s_2}}}  - \entrp{X},$$ which proves the first inequality. Doing the same for $\entrp{X | Z^{\bar{s}}}$ proves the second inequality. 
\end{proof}

Each of the quantities in Theorem~\ref{PartitionDepthLB} is a lower bound on the expected reward by itself. As for the upper bound, the Entropy of each of the child sets of a given set, is an upper bound on the Entropy of that given parent set.
\begin{theorem} \label{PartitionDepthUB} For the set $Z^{s}$, and its children $Z^{s_1},Z^{s_2}$ the following holds:
\begin{equation}
	\entrp{X \vert Z^{s}} \leq \entrp{X \vert Z^{s_1}} \land \entrp{X \vert Z^{s}} \leq \entrp{X \vert Z^{s_2}}. 
\end{equation}
\end{theorem}
\begin{proof}
	This follows directly from the conditioning argument.
\end{proof}

We can use the above theorems to perform further partitioning of the measurement sets. Theorem~\ref{PartitionDepthLB} shows that we can mix different partition depths for the lower bound, while Theorem~\ref{PartitionDepthUB} shows that any node in the partition tree can be used to form upper bound, see partition tree Fig.~\ref{fig: PartitionTreeBounds} for example. 

\begin{figure} [!h] 
	\hspace*{35pt}
	\includegraphics[scale=0.5]{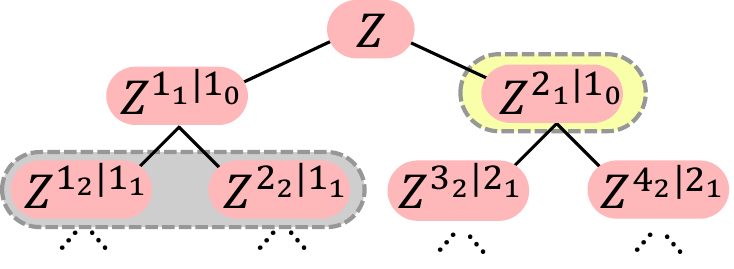}
	\scriptsize
	\caption{\scriptsize Any combination of children nodes, that their union equals to the parent node, can make up a lower bound on that parent node. Any child node at any depth can make up an upper bound on a parent node. For instance, the node highlighted in yellow can form an upper bound, while the nodes highlighted in gray and yellow can form a lower bound.}
	\normalsize
	\vspace{6pt}
	\label{fig: PartitionTreeBounds}
\end{figure} 

We also note that we have the ability to adaptively change those bounds, by moving between partitioning levels, as well as by moving measurements from one set to its compliment. In the next section we show the convergence of the bounds, so by adaptively changing the bounds we also control how tight they are. 

\subsection{Analysis of the Bounds} \label{subsec: Analysis}
In this section, we analyze the properties of the derived bounds. We look at the convergence and monotonicity of the bounds, both as a function of the sets, for a given partition depth, and as a function of the partition depth. We then examine the computational complexity of obtaining the bounds.  
\subsubsection*{Convergence} 
For a given partition depth of $d\in [1,\log_2m]$, we show that the bounds converge to the bounds of the parent depth, $d-1$. In particular, when the partition depth is $1$ the bounds converge to the original expected reward.    
For $Z^s \cup Z^{\bar s} \subseteq Z$, the upper bound converges when we add variables to the set $Z^s$, while the lower bound converges when we remove variables from the set $Z^{\bar s}$ and add them to the set $Z^s$. We show the proof for the first partition depth, but it is valid for any arbitrary depth.
\begin{theorem}
	If $Z^s \rightarrow Z$ and $Z^{\bar s} \rightarrow \emptyset$, then $\entrp{X | Z^{s}} \rightarrow \entrp{X | Z}$ and $g(Z^{s} ,Z^{\bar{s}}) \rightarrow \entrp{X | Z}$. 
\end{theorem}
\begin{proof}
	WLOG, assuming $Z^s\in \mathbb R^n$ and $Z^{\bar s}\in \mathbb R^{m-n}$ as in \eqref{MasurementPartitioning}, we start with the upper bound,
	$\mathcal{UB} = \entrp{X | Z^s}.$
	Using the conditioning argument, $\entrp{X | Z^s\cup \{Z^{n+1}\}} \leq \entrp{X | Z^s}$ where $ Z^{n+1}\subseteq Z^{\bar s}$ but $ Z^{n+1} \not\subset Z^{s}$. We continue adding variables to the set $Z^s$ in this way, until $Z^s \cup \{Z^{n+1}\} \cup \hdots \cup  \{Z^{m}\} =Z$ and $\entrp{X | Z^s \cup \{Z^{n+1}\} \cup \hdots \cup  \{Z^{m}\}} = \entrp{X | Z}$. \\
	As for the lower bound: $ \mathcal{LB} = g(Z^{s} ,Z^{\bar{s}}) = \entrp{X | Z^s} +\entrp{X | Z^{\bar s}} - \entrp{X}.$
	 We use the same argument, while working in the opposite direction with $Z^{\bar s}$: $\entrp{X | Z^{\bar s}} \leq \entrp{X \vert  Z^{\bar s} \setminus \{Z^{n+1}\}}$, where $Z^{n+1}\subseteq Z^{\bar s}$. We continue removing variables from the set $Z^{\bar s}$ in this way until $Z^{\bar s} \setminus \{Z^{n+1}\} \setminus \hdots \setminus \{Z^{m}\}= \emptyset$ and $\entrp{X | \emptyset} = \entrp{X}$. Plugging into the expression for $\mathcal{LB}$ we get $\entrp{X | Z} +\entrp{X | \emptyset} - \entrp{X}=\entrp{X | Z}$  
\end{proof}

\subsubsection*{Monotonicity} From the conditioning argument, we can see that $\mathcal{UB}$ monotonically converges to the expected reward, or in different words, the Entropy is monotone in the size of the measurement set. This holds true both across different partitioning depths and within a given partitioning depth. However, we cannot say the same about $\mathcal{LB}$ since it depends on the change of both $\entrp{X | Z^{s}}$ and $\entrp{X | Z^{\bar{s}}}$. We always move measurements from one set to the other, such that these two quantities change in opposition to one another. We cannot say a-priori what change is greater and thus cannot deduce monotonicity. In fact, it can be proven that this bound is non-monotone using the sub-modularity of the Entropy in the measurements, but it is beyond the scope of this work. 
\subsubsection*{Computational Complexity}
It is only rational to use the bounds instead of the full calculation of the expected reward function, if computing the bounds is cheaper computationally. We are now going to compare these two calculations.

For general distributions, we can compare the calculation as a function of the observation space size. The baseline calculation involves evaluating the following quantities (omitting history and actions, see lemma \ref{lem:Factorization}): 
\begin{equation}
	 \entrp{Z^{s} \vert X},\entrp{Z^{\bar{s}} \vert X},\entrp{X},\entrp{Z^{s},Z^{\bar{s}}}.	 
\end{equation}
Using the bounds we need to evaluate (see theorems \ref{th: LB} and \ref{th: UB}):
\begin{equation} 
	\entrp{Z^{s} \vert X}
	,\entrp{Z^{\bar{s}} \vert X} ,\entrp{Z^{\bar{s}}},\entrp{Z^{s}},\entrp{X}.
\end{equation}
Since that the prior Entropy is not a function of the action or expected measurements, we can calculate it once for each planning session.
Overall, the difference between the expected reward and the bounds boils down to the difference between the joint versus the marginal Entropy of the measurements, i.e. $\entrp{Z^{s},Z^{\bar{s}}}$ versus $\entrp{Z^{\bar{s}}}$ and $\entrp{Z^{s}}$. The baseline calculation is:
\begin{multline} \label{joint Entropy baseline}
	 \mathcal{H}\left(Z\right) \! =	\!  - \! \int_{Z^{s}}\! \int_{Z^{\bar{s}}} \! \prob{Z^{s},Z^{\bar{s}}} \log\prob{Z^{s},Z^{\bar{s}}} dZ^{s} dZ^{\bar{s}}. 
	\end{multline} 
Assuming that the observation space is finite and countable, the cost of evaluating the integral terms is a function of the random variables. 
Evaluating \eqref{joint Entropy baseline} is of the order of $O(| \mathcal{Z}^{s} | | \mathcal{Z}^{\bar{s}} |)$, while evaluating the simplified terms $\mathcal{H}\left(Z^{s}\right),\mathcal{H}\left(Z^{\bar{s}}\right)$, is of the order of $O(| \mathcal{Z}^{s} |+ | \mathcal{Z}^{\bar{s}} |)$.  $\mathcal{Z}^s$ and $\mathcal{Z}^{\bar{s}}$ are entirely defined by the measurement model, see example on section \ref{sec:Partitioning}

Note that in practice we do not have access to the joint or marginal distribution over the measurements, only the measurement model. Marginalization over the state yields:
\begin{multline} 
	 \mathcal{H}\left(Z\right) =
	 -\int_{Z} \int_{X} \prob{Z,X} \log \int_{X'} \prob{Z,X'} dX' dZ dX  \\ =\!\!
	-\int_{Z} \int_{X} \prob{Z\vert X} \prob{X} \log \int_{X'} \prob{Z\vert X'} \prob{X'}dX' dZ dX  
	\end{multline} 
The simplification is of the same factor, only now it is magnified by the state vector size to the second power; $O(| \mathcal{Z}^{s} | | \mathcal{Z}^{\bar{s}} ||\mathcal{X}^2|)$ for the baseline, and $O((| \mathcal{Z}^{s} | +| \mathcal{Z}^{\bar{s}} |)|\mathcal{X}^2|)$ for the bounds.
\subsection{High Dimensional State}
In this section we discuss what the measurement simplification looks like when the state space is high dimensional. Specifically, we consider an active SLAM formulation as in Section \ref{active slam}. Utilizing data association information from \eqref{eq: BeliefModelFactorization} and \eqref{eq: InvolvedState}, we can further simplify the bounds \eqref{eq: RewardUB} and  \eqref{eq: RewardLB}: 
	\begin{multline} 
		\mathcal{LB}\triangleq\mathcal{H}\left(Z^{s}\vert X^{\text{inv}_s}\right)
	+\mathcal{H}\left(Z^{\bar{s}}\vert X^{\text{inv}_{\bar{s}}}\right) \\ -\mathcal{H}\left(Z^{s}\right)-\mathcal{H}\left(Z^{\bar{s}}\right)+\mathcal{H}\left(X\right),
	\end{multline}
\begin{equation} 
	\mathcal{UB}\triangleq \mathcal{H}\left(Z^{s}\vert X^{\text{inv}_s}\right)+\mathcal{H}\left(X\right)-\mathcal{H}\left(Z^{s}\right),
	\end{equation}
where $X^{\text{inv}_s}$ and $X^{\text{inv}_{\bar{s}}}$, are the states involved in the random measurements $Z^{s}$ and $Z^{\bar{s}}$, respectively, as determined by $\beta$. 
\subsection{Gaussian Belief}

The bounds shown in the previous section, hold for general belief distributions. In this section we develop a specific form of the bounds, considering Gaussian distributions and a high-dimensional state in the context of active SLAM. 

In the case of Gaussian distributions, we can replace the Entropy term in \eqref{eq:ExpEntropyGeneral} for its closed form expression, combining \eqref{GaussianEntropy} and \eqref{InfoMatUpdate}.
For a lookahead step $i\in[k+1,k+\ell]$ the expected Entropy can be expressed as: 
\begin{equation} \label{eq: EntrGaussianExpected} 
	\expt{Z_{k+1:i}}{\entrp{X_{i} | h_{i}}} = 
	C -\frac{1}{2} \expt{Z_{k+1:i}}{\ln \left | \Lambda_k^{\text{Aug}}+A_i(z)^T A_i(z)  \right |},
\end{equation}
where $ A_{i}(z)=\tilde{A}_{i}(z_{k+1:i})$, is the collective Jacobian for the entire planning horizon as a function of the measurements, and $C \triangleq N\ln(2\pi e)$. 

 For convenience we define: $f(\Lambda, A) \triangleq | \Lambda+A^T A |$. Since we are interested in the Jacobian of the measurement but not the motion factors, we choose to present the expected Entropy in the following form:
 \begin{equation}
 	\expt{Z_{k+1:i}}{\entrp{X_{i} | h_{i}}} = 
	C -\frac{1}{2} \expt{Z_{k+1:i}}{\ln f\left(\Lambda^{\textsc{Aug}-}_{k},A_{i}(z)\right)},
 \end{equation}
 where $\Lambda^{\textsc{Aug}-}_{k}$ is the propagated belief, i.e. without taking into account measurements until time step $i$, augmented with zeros and the measurements vector defined in \eqref{MasurementVector} becomes $Z_{k+1:i}$. 
 
 Utilizing \eqref{eq:RewardLBParam}, \eqref{eq:RewardUBParam} and \eqref{eq: EntrGaussianExpected}, we are ready to present the expected reward bounds for the Gaussian case. The bounds apply in expectation, and are more efficient to calculate than State-of-the-art (SOTA) methods per sample of measurements under the expectation operator.
 \begin{figure} [!h] 
	\hspace*{50pt}
	\includegraphics[scale=0.5]{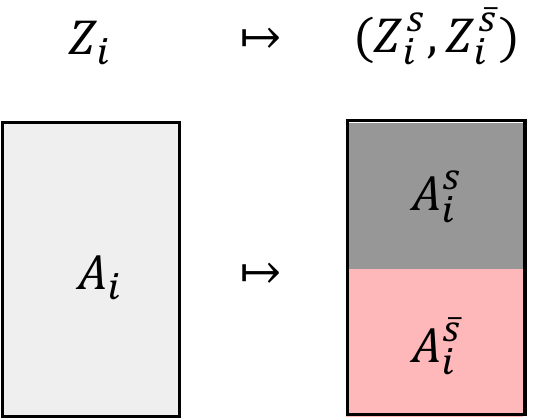}
	\scriptsize
	\caption{\scriptsize Illustration of a possible partitioning of a measurement collective Jacobian, given some partitioning of $Z_i$.}
	\normalsize
	\vspace{6pt}
	\label{fig: JacobianPartition}
\end{figure} 
 \begin{theorem}
	For the $i$th look ahead step, the expected Entropy of a Gaussian belief can be bounded by: 
\begin{align} \label{eq: RewardLBGaussian}
		\mathcal{LB} &= C- \!\!
		 \frac{1}{2} \expt{Z_{k+1:i}}{\ln \frac{f\left(\Lambda^{\textsc{Aug}-}_{k},A^s_{i}\right) \! \cdot \!\! f\left(\Lambda^{\textsc{Aug}-}_{k},A^{\bar s}_{i}\right)
			}{\vert\Lambda^{\textsc{Aug}-}_{k}\vert }},
		\\
		\mathcal{UB}&= C -\frac{1}{2} \expt{Z_{k+1:i}}{\ln f\left(\Lambda^{\textsc{Aug}-}_{k},A^s_{i}\right)},
		 \label{eq: RewardUBGaussian}
\end{align}
\end{theorem}
where $A^s_{i}$ and $A^{\bar{s}}_{i}$ are the Jacobian rows associated with the measurements $Z^s_{i}$ and $Z^{\bar{s}}_{i}$, respectively, see Fig.~\ref{fig: JacobianPartition}. 

\begin{proof}
	We have obtained a closed form expression \eqref{eq: EntrGaussianExpected} for the expected Entropy. If we plug in these expressions into \eqref{eq:RewardLBParam} we get: 
\begin{align*}
	\mathcal{H} \left(X_{i} | Z_{i} \right) &\geq 
	C -\frac{1}{2} \expt{Z_{k+1:i}}{\ln f\left(\Lambda^{\textsc{Aug}-}_{k},A^s_{i}\right)} \\
	&+C -\frac{1}{2} \expt{Z_{k+1:i}}{\ln f\left(\Lambda^{\textsc{Aug}-}_{k},A^{\bar s}_{i}\right)}  \\ 
	&- C +\frac{1}{2} \expt{Z_{k+1:i}}{\ln \left \vert\Lambda^{\textsc{Aug}-}_{k}  \right \vert} .
\end{align*}
In a similar way, using \eqref{eq:RewardUBParam} we obtain the upper bound.
\end{proof}
\subsubsection*{Estimation of Expected Entropy} \label{sssec: Estimation}
An empirical expectation is generally calculated by sampling the observation model; for each such sampled observation a posterior distribution of the belief needs to be obtained. The latter usually involves solving a non-linear optimization problem via some iterative linearization method. For each posterior belief, one can calculate the corresponding reward bounds using \eqref{eq: RewardLBGaussian} and \eqref{eq: RewardUBGaussian}, where the expectation operator is approximated by samples. In this case, the performance guarantees are asymptotic, since one assumption that was used to obtain the bounds is only asymptotically valid, namely, the non-negativity of the Mutual Information between measurements. We note that it is possible to formulate non-asymptotic guarantees in this case, but it is outside of the scope of this work and we leave it for future work.

A different approach, would be to make a common assumption for the optimization process within planning, by taking a single iteration step. In the context of planning, this is usually a good approximation for the linearization point when the state prior is informative, and is accurate when the measurement model uses a linear function or for Maximum Likelihood estimation for the measurements \cite{Platt10rss, Indelman15ijrr}. Under said assumption, the Jacobian is only a function of $\beta$ and not of any particular measurement realization.  When the Jacobian is independent of the actual measurement values, we can drop the expectation operator from \eqref{eq: EntrGaussianExpected} and still represent the expected Entropy: 
 \begin{multline}\label{eq:EntrGaussian} 
	\mathcal{H} \left(X_{i} | Z_{k+1:i} \right) = 
	 C-\frac{1}{2}(\ln \left \vert \Lambda_k^{\text{Aug}}+(A_{i})^T\cdot A_{i} \right \vert).
\end{multline}
In this case, the bounds in \eqref{eq: RewardLBGaussian} and \eqref{eq: RewardUBGaussian}, as well as other SOTA calculations should be computed only once per action, i.e., without the expectation operator. \\
\subsubsection*{Methods for Determinant Calculation} \label{sssection: DetMethods}
The computational cost of the bounds in the Gaussian case, is the cost of evaluating the appropriate determinants, $\vert\Lambda^{\textsc{Aug}-}_{k}+(A^s_{i})^T A^s_{i} \vert$ and $ \vert\Lambda^{\textsc{Aug}-}_{k}+(A^{\bar{s}}_{i})^T A^{\bar{s}}_{i}\vert$ . At a first glance, it is not clear that evaluating the bounds is indeed more efficient than evaluating the expected reward, we will now show one method which is, in fact, efficient. 

There are three main methods to calculating the Entropy of the posterior belief: calculating the determinant of the posterior information matrix (denoted to as baseline), calculating the determinant of the posterior information matrix in its square root form (denoted as $R$), and using the Augmented Matrix Determinant Lemma (rAMDL) \cite{Kopitkov17ijrr, Kopitkov19ijrr}. In essence, rAMDL utilizes the well known Matrix Determinant Lemma, combined with some clever calculation re-use, to efficiently evaluate the determinant of the posterior information matrix. It requires a one-time calculation of some specific covariance entries, and its general form is as follows:
\begin{equation} \label{eq: rAMDLgeneral}
	| \Lambda+A^T A | =  | \Lambda | \cdot| \Delta | \cdot  \left \vert (A_\text{new})^T\cdot \Delta^{-1} \cdot A_\text{new} \right \vert,
\end{equation} 
where $\Delta=I_m+A_\text{old} \cdot \Sigma \cdot (A_\text{old})^T$, $\Sigma$ is the prior covariance matrix, and $A_\text{old}$ and $A_\text{new}$ are the blocks of the Jacobian matrix $A$, with respect to states at planning time (old) and states added by future actions (new), i.e., $A=[A_{old}, A_{new}]$\\
Applying measurement partitioning to the baseline method is not efficient: the baseline cost is $O(n^3)$, where $n$ is the dimension of $\Lambda$. Since $n$ remains the same given a partition, the overall calculation would be worse than $O(n^3)$. 

We choose to apply measurement partitioning to rAMDL and not $R$ for two reasons: Firstly, it has been shown that rAMDL is faster than $R$ within planning \cite{Kopitkov17ijrr, Kopitkov19ijrr}. Secondly, applying measurement partitioning in $R$ does not make sense -  although calculating the determinants in $R$ is $O(n)$, most of the computational burden lies in the update of the posterior $R$ matrix. This update involves a QR decomposition and by partitioning we split it into two separate QR decompositions that on average, are less efficient than the original one. We are then left with the most efficient method, which is rAMDL.

\subsubsection*{Computational Complexity} \label{sssection: GaussianCC}
One main insight, is that the cost of calculating the determinant using rAMDL (not including the one-time calculation which is common across all actions), is $O(m^3)$, where $m$ is the dimension of $A$. In a typical high-dimensional scenario, $m \ll n$, and by partitioning the measurements we reduce the dimension of $A$ even further so it becomes even more efficient. 

Applying \eqref{eq: rAMDLgeneral} to the determinants required for \eqref{eq: RewardLBGaussian} and \eqref{eq: RewardUBGaussian} results in:
\begin{equation*}
	\vert\Lambda^{\textsc{Aug}-}_{k}+(A^s_{i})^T A^s_{i} \vert =  \vert \Lambda_{k}  \vert \cdot \vert \Delta_k^s \vert \cdot  \left \vert (A^s_\text{new})^T (\Delta^s_{k})^{-1}  A^s_\text{new} \right \vert,
\end{equation*}
where $\Delta_k^s=I_m+A^s_\text{old} \cdot \Sigma_k \cdot (A^s_\text{old})^T$, and $A^s_\text{old}$ and $A^s_\text{new}$ are the old and new blocks of the matrix $A_s$, i.e., $A^s=[A^s_{old}, A^s_{new}]$. In a similar way we apply rAMDL to $ \vert\Lambda^{\textsc{Aug}-}_{k}+(A^{\bar{s}}_{i})^T A^{\bar{s}}_{i}\vert$. 

The most efficient partitioning, for the first depth, is when $s=\bar{s}$, so the resulting matrices are of rank $\frac{m}{2}$. The cost of evaluating each of the determinants for the partitioned Jacobians is $O(\frac{m^3}{8})$. Compared to rAMDL, the cost is reduced from $O(m^3)$ to $O(\frac{m^3}{4})$ by using the bounds.     

%% file: 05-results.tex
\section{RESULTS}
In this section, we start off by demonstrating the properties of the bounds as presented in \ref{subsec: Analysis}, for a typical active SLAM scenario. We then show that compared to the other SOTA method our approach is faster while achieving similar reward, both in a simulated scenario and in real-world experiment. 
\subsection{Setup}
The setting considered is a planar high-dimensional SLAM, where the state includes past poses, current poses, and landmarks. The belief over the state is normally distributed and the reward function is differential Entropy. The action space is continuous, and each pose and landmark has 3 and 2 degrees of freedom, respectively. Thus, for a planning session at time instant $k$ and planning horizon $\ell$,  the joint state has the dimensionality of  $3(k+\ell)+2L$, where $L$ is the number of mapped landmarks. As mentioned in section \ref{sssec: Estimation}, in our simulations we generate  Maximum-Likelihood observations, although this is not a limitation of our approach. In all scenarios, the agent uses its sensor to map an unknown environment. In the simulation we use a range and bearing sensor, the real experiments uses cameras and runs visual SLAM. After an initial mapping procedure, a prior belief containing poses and landmarks is formed. The agent then starts a planning session, where the objective is to reach a certain goal while minimizing the uncertainty over the high-dimensional state at the last planning step. A typical prior belief is illustrated in Figure \ref{fig: Uncertainty}.  

GTSAM 4.1.0 and Python 3.9.7 were used for the simulated scenarios, running on Ubuntu 18.04 and AMD Ryzen 7 3700X 8-Core Processor. For the real-world experiment GTSAM 4.1.0 and Python 3.8 were used, running on Ubuntu 20.04 and Intel(R) Xeon(R) CPU E5-1620.

\begin{figure} [!h] 
	\includegraphics[width=\columnwidth]{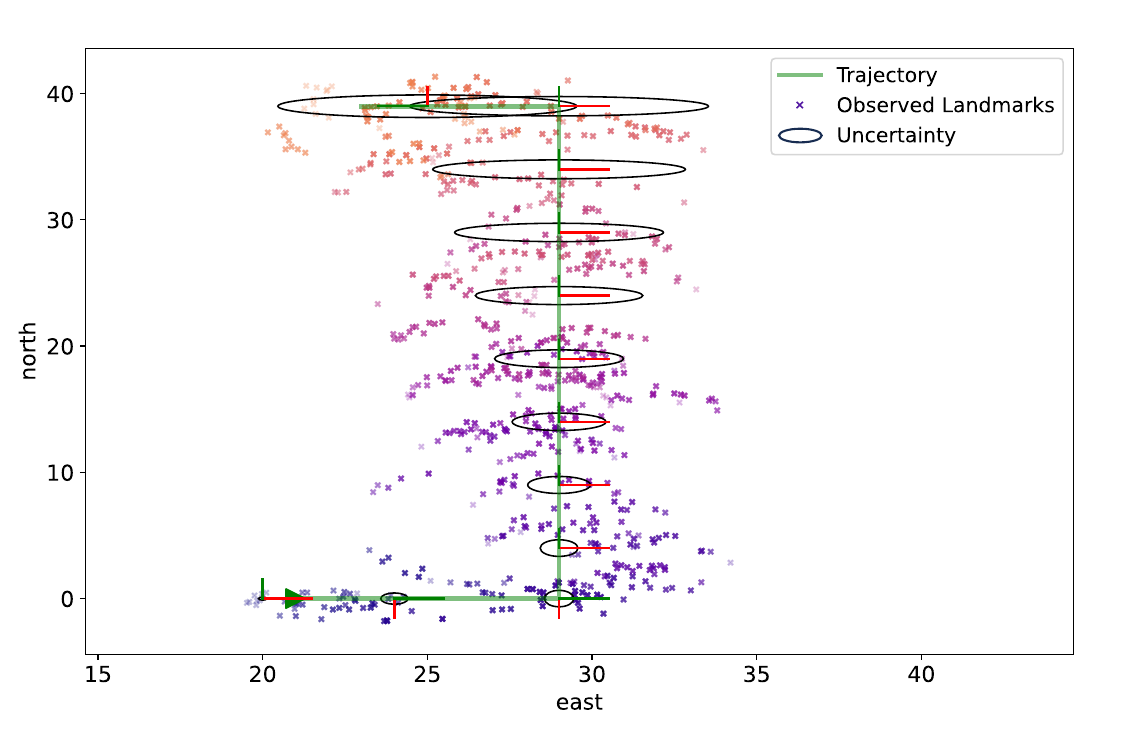}
	\scriptsize
	\caption{\scriptsize Illustration of the prior belief showing a subset of the state landmarks. The uncertainty associated with the joint covariance matrix is shown for every fifth pose. The time of each observed landmark is encoded in color - from a dark one at the beginning to light at the end. }
	\normalsize
	\vspace{6pt}
	\label{fig: Uncertainty}
\end{figure} 

\subsection{Bounds Analysis}
\subsubsection*{Inter Depth Properties}
Figure \ref{fig: BoundsConverge} shows the properties of each of the bounds, as individual measurements are moved from one partition to the other, specifically, as $Z^s \rightarrow Z$ and $Z^{\bar s} \rightarrow \emptyset$. We can see that both bounds converge to the actual expected Entropy, the upper bounds does so monotonically while the lower bounds does not, as expected.
\begin{figure} [!h] 
	\includegraphics[width=\columnwidth]{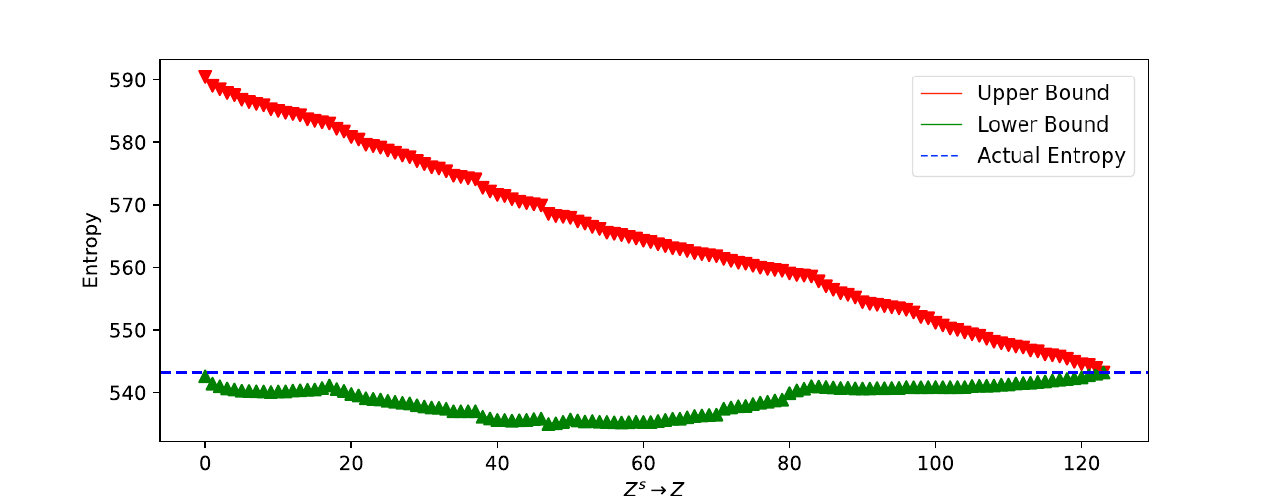}
	\scriptsize
	\caption{\scriptsize Empirical demonstration of the bounds' behavior, for $Z^s \rightarrow Z$.}
	\normalsize
	\vspace{6pt}
	\label{fig: BoundsConverge}
\end{figure} 
As proposed in Section \ref{subsec: Analysis}, the lower bound does not converge monotonically because it is a function of the mutual information between the measurement partitions. Since the assignment to partitions is random, so is the value of the mutual information between them.

\subsubsection*{Intra Depth Properties}
We show the properties of the bounds, as we go deeper in the partition tree. Figure \ref{fig: BoundsConvergeDepth} shows, as expected, that the lower we go down the partition tree, the efficiency of the bounds comes at the expanse of their tightness. We can see that both bounds monotonically converge to the real expected reward as we go up the partition tree.

\begin{figure} [!h] 
	\includegraphics[width=\columnwidth]{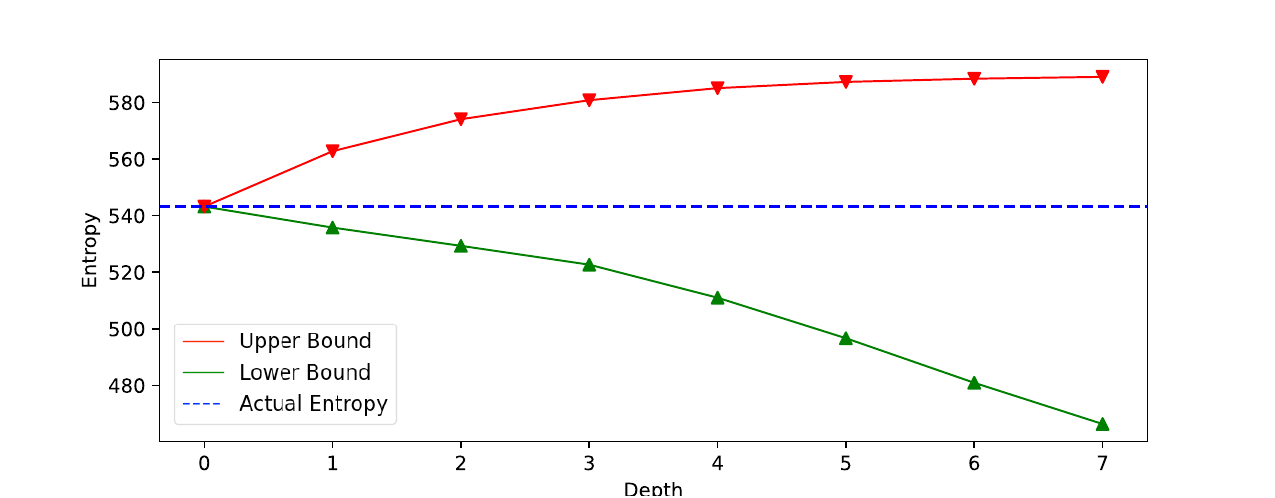}
	\scriptsize
	\caption{\scriptsize Empirical demonstration of the bounds' behavior, as $d \rightarrow \log_2 m$.}
	\normalsize
	\vspace{6pt}
	\label{fig: BoundsConvergeDepth}
\end{figure} 

\subsubsection*{Correlation and Tightness}
Here we demonstrate the usefulness of the bounds in the general case, by presenting the expected reward bounds, for a set of myopic random actions. Looking at figure \ref{fig: RandomActionBounds}, we can see that although some bounds overlap, we can safely prune around 60 percent of the actions which are sub-optimal.
 
 \begin{figure} [!h] 
	\includegraphics[width=\columnwidth]{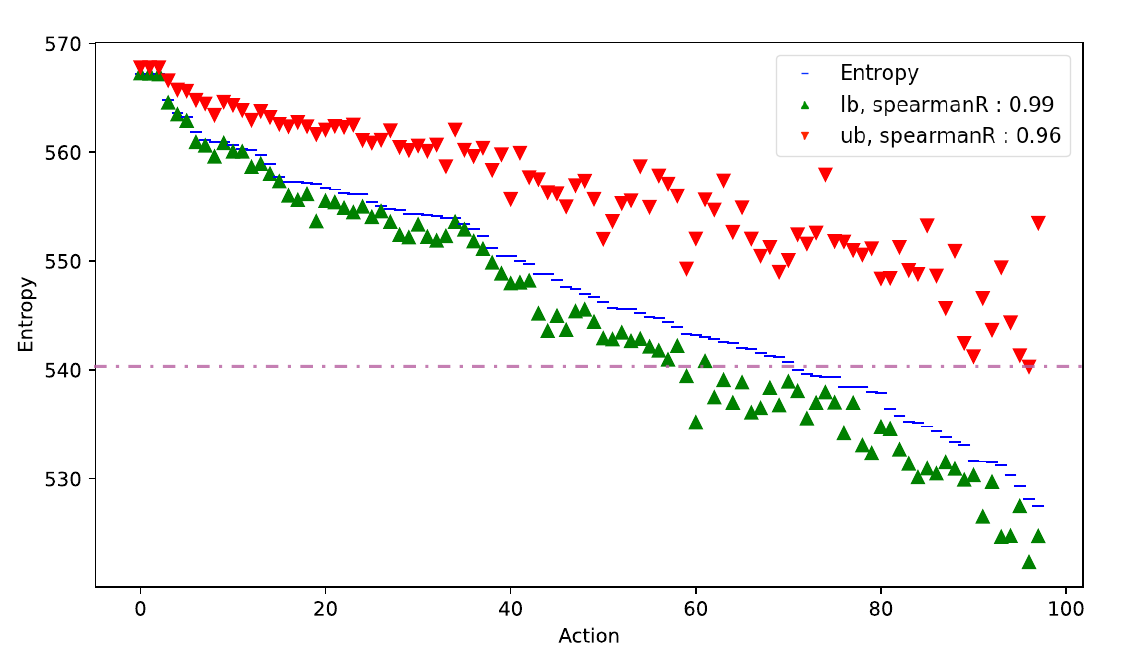}
	\scriptsize
	\caption{\scriptsize Expected Entropy bounds for each sampled action. All actions whose lower bound is above the purple dashed line are sub-optimal and can be pruned safely.}
	\normalsize
	\vspace{6pt}
	\label{fig: RandomActionBounds}
\end{figure} 
 
 We can observe that the lower bound is very strongly correlated to the actual reward while the upper bound shows slightly weaker correlation to the reward. Again, a possible explanation for this behavior is the random assignment of the individual observations to their corresponding set. Since the lower bound accounts for both sets, it is only a function of their mutual information, while the upper bound is a function of this random assignment. 
  
 Next we compare two specific actions, taken from the previous random action set, by showing their respective hierarchical bounds. Figure \ref{fig: BoundsDepth} shows, that for actions that lead to rewards which are distinct, it is possible to select the optimal action based on a very efficient calculation of the bounds. Specifically, we compare an action which explores a new part of the map, to another which re-observes previously seen landmarks.
 
  \begin{figure} [!h] 
	\includegraphics[width=\columnwidth]{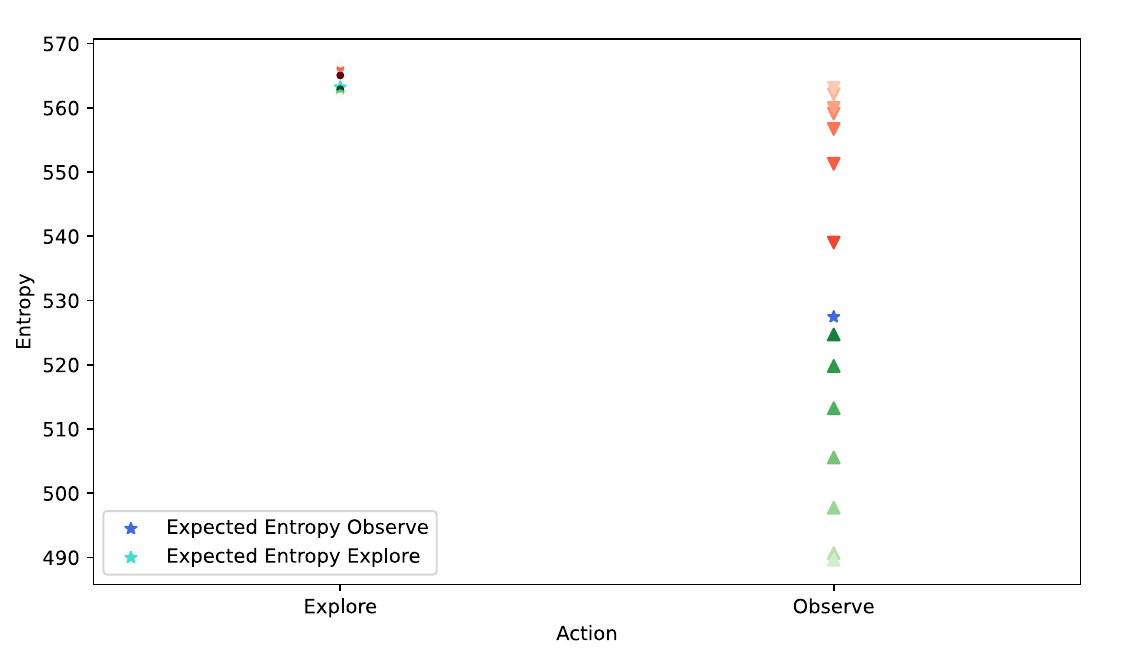}
	\scriptsize
	\caption{\scriptsize Hierarchical bounds for two distinct actions "Explore" and "Observe".  The lighter the color, the more efficient the bounds are.}
	\normalsize
	\vspace{6pt}
	\label{fig: BoundsDepth}
\end{figure} 

\subsection{Planning Using Bounds}
\subsubsection{Simulation} \label{sssection: Simulation}
We now demonstrate the use of the bounds in a simulated SLAM scenario using the GTSAM library and a \emph{Probabilistic Roadmap} (PRM) generator \cite{Kavraki96tra}, the goal is to demonstrate a typical use-case in active SLAM setting.  We start the planning session with a prior belief as in figure \ref{fig: Uncertainty}. We then use the PRM to randomly generate a set of candidate paths, shown in figure \ref{fig: PRM}. As seen in figure \ref{fig: IllustrationScenario1}, the paths start at the green dot and the goal shown in red, each path consisting of around 12 actions. The PRM samples are drawn from a uniform continuous distribution. The actions of a given path follow the trajectory dictated by the PRM. For each path from a set of the shortest ones, we evaluate the expected reward at the goal, and its bounds, then choose the optimal one. We emphasize that our approach is applicable also to other methods for generating candidate action sequences and not limited to the above-described specific method.

  \begin{figure} [!h] 
	\includegraphics[width=\columnwidth]{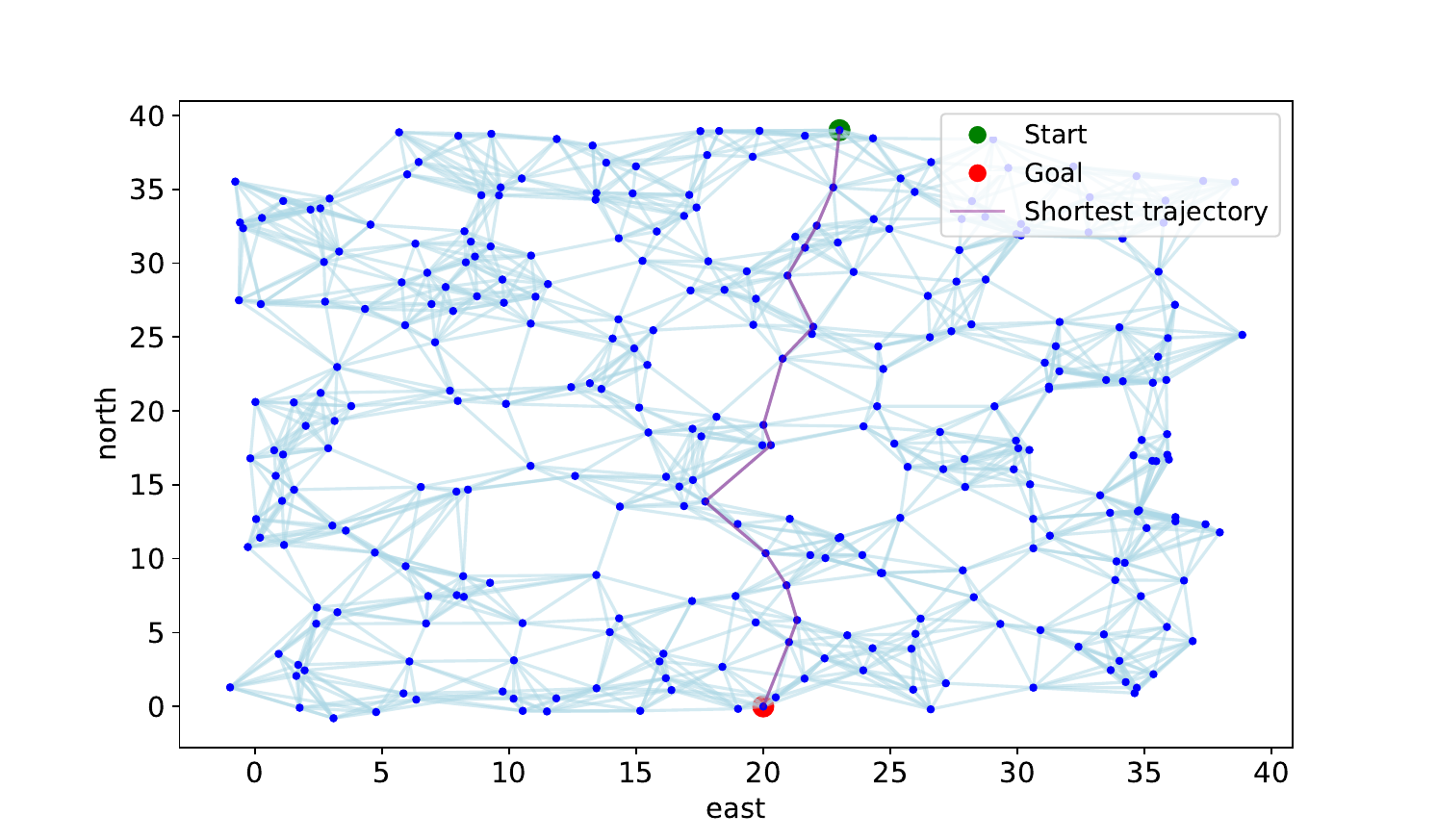}
	\scriptsize
	\caption{\scriptsize PRM generator is used to construct a viable path from the agent's pose at the start of the planning session to the desired goal. We then select a number of the shortest ones.}
	\normalsize
	\vspace{6pt}
	\label{fig: PRM}
\end{figure} 

\begin{figure} [!h] 
	\includegraphics[width=\columnwidth]{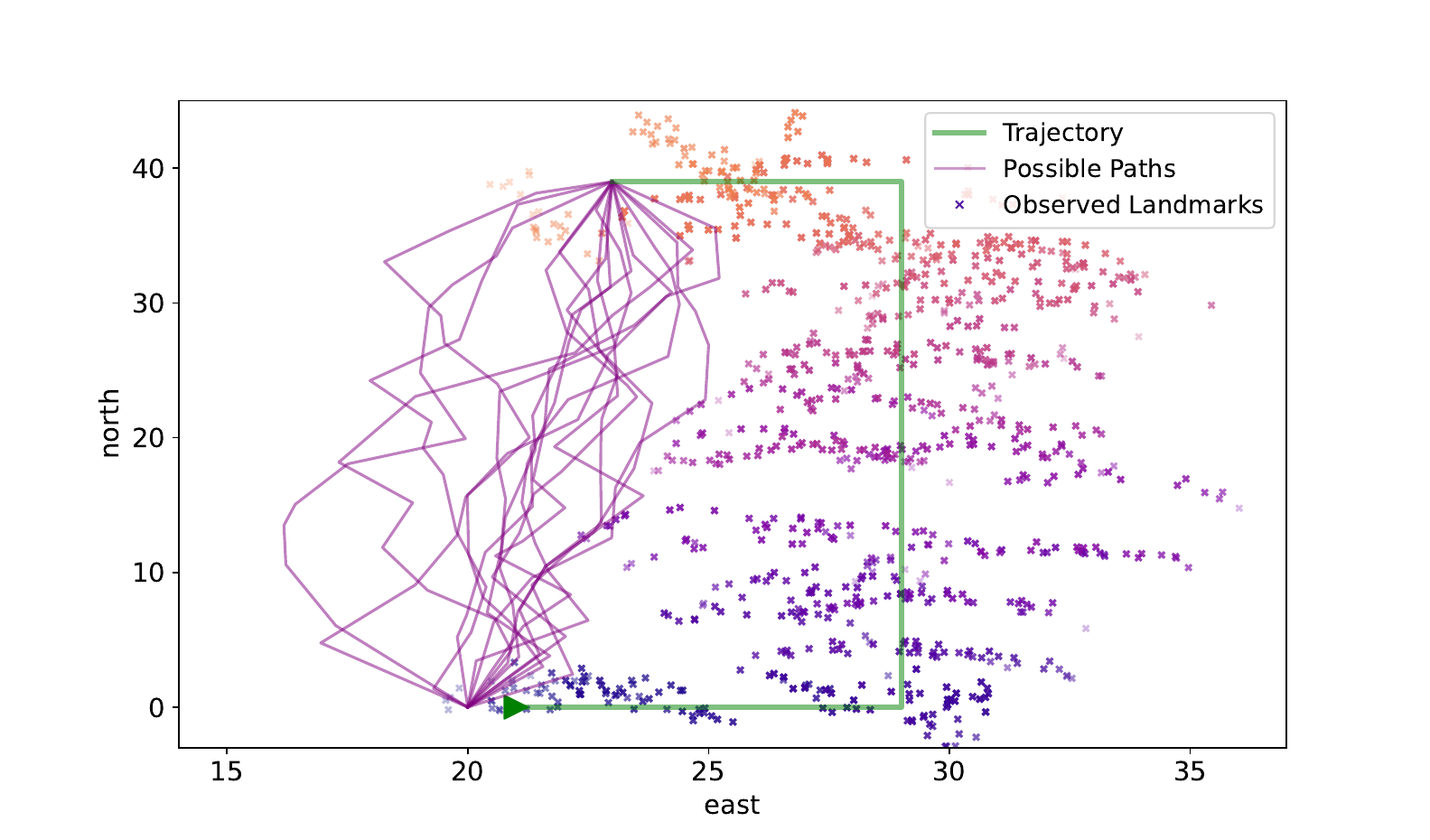}
	\scriptsize
	\caption{\scriptsize Illustration of the simulated scenario, showing a subset of the landmarks and paths.}
	\normalsize
	\vspace{6pt}
	\label{fig: IllustrationScenario1}
\end{figure} 

We start by comparing our proposed method using Measurement Partitioning (MP) to rAMDL. Since our method uses rAMDL to compute the determinants required for the bounds, many of the calculations are common and  are not accounted for in the comparison. When compared to iSAM2, these will be accounted for.  We compared the performance for three different scenarios, while two of them utilize re-planning. In a re-planning scenario, the agent chooses an optimal path, takes the first action, re-draws paths to the goal from the new pose. The number of re-planning steps is 5, since the goal is fixed, each consecutive trajectory is shorter than its predecessor, see Fig \ref{fig: RPScenario} for details.  

From Table \ref{tab: MPVsrAMDL} we can see that the speed-up increases with the number of factors. That is expected since the speed-up is correlated with the number of Jacobian rows. Since most of the prior factors relate to landmarks, a higher number of factors results in a higher number of observations. We can also see that the speed-up is approaching the theoretical value of $O(\frac{m}{4})$, shown in Section \ref{sssection: GaussianCC}. This is because the higher the number of observations, the smaller the impact of the one-time overhead on the total planning time. 
Similar to all of the following scenarios, we use the bound to prune sub-optimal paths; when the bounds overlap we choose the path that has the lowest lower bound and bound the possible loss we might incur. For example, for the  third scenario in table \ref{tab: MPVsrAMDL}, the bound on the expected loss was 109.578 and the optimal expected reward was 5215.999. The ratio of loss divided by the optimal reward was 0.021, such that in the worst case, the trajectory chosen would set us off from the optimal one by two percent. Although unknowable in planning, in this specific case, as in many others, the lowest lower bound  indeed corresponded the optimal expected reward, and the loss in practice was equal to zero. 
 \begin{figure} [!h] 
	\includegraphics[width=\columnwidth]{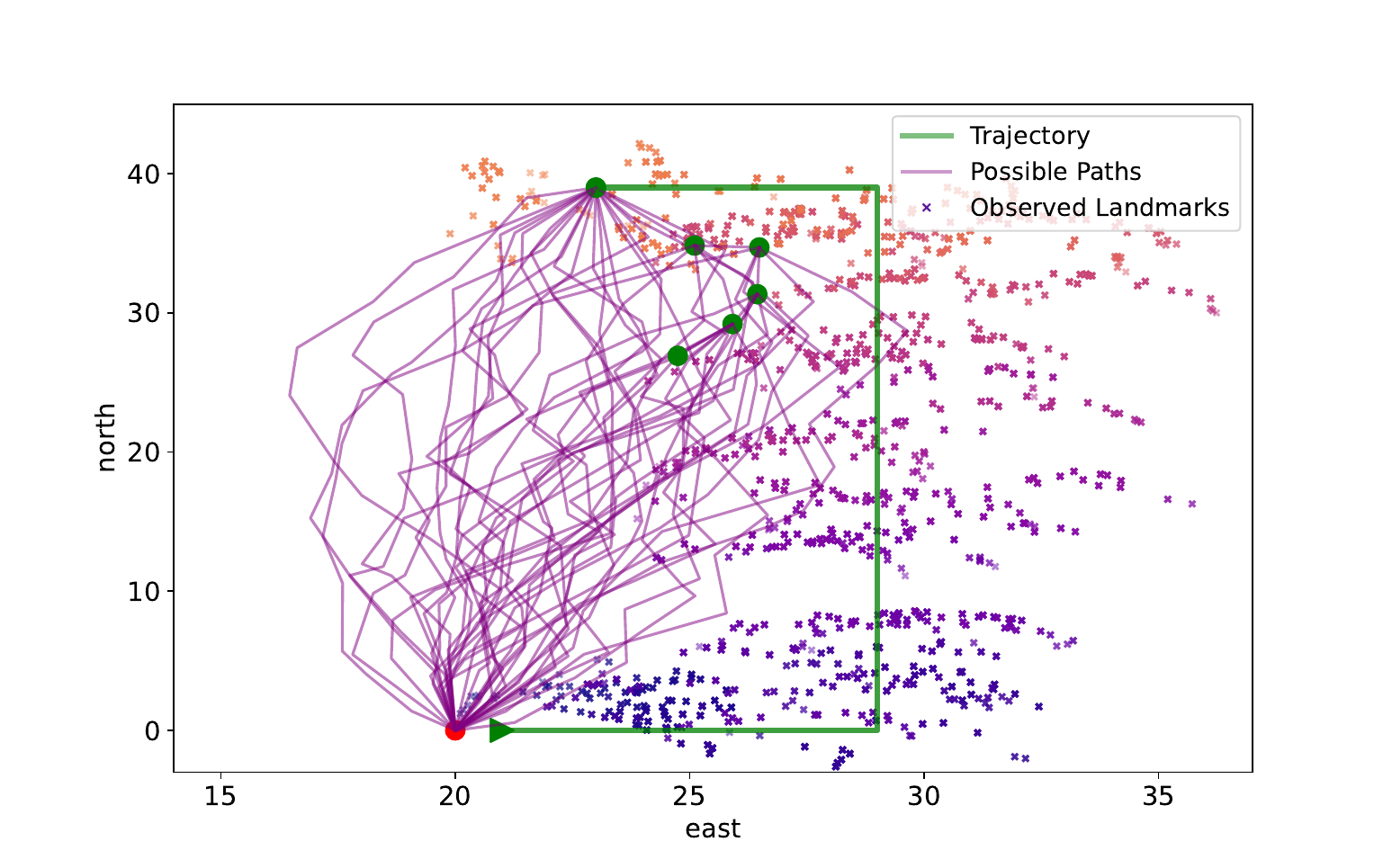}
	\scriptsize
	\caption{\scriptsize Re-planning scenario, each green dot represents the new pose for planning, the goal is represented by a red dot.}
	\normalsize
	\vspace{6pt}
	\label{fig: RPScenario}
\end{figure} 
    
\begin{table}[h!]
  \begin{center}
     
    \begin{tabular}{c | c | c | c | c } % <-- Alignments: 1st column left, 2nd middle and 3rd right, with vertical lines in between
    
      \# Paths & \# Factors & RP & rAMDL &MP (ours)\\
      \hline
      100 & 2956 & No & $\boldmath{11.521 \pm 0.537}$ & $\boldsymbol{6.888 \pm 0.155}$ \\  
      100 & 2956 & Yes & $\boldmath{24.636 \pm 1.381}$ & $\boldsymbol{11.758 \pm 0.372}$ \\
      100 & 5904 & Yes & $\boldmath{84.376 \pm 14.458}$ & $\boldsymbol{32.069 \pm 4.913}$
     
    \end{tabular}
    \caption{    \label{tab: MPVsrAMDL}Total planning time in seconds (lower is better), not including all common one-time calculations. For all scenarios the same best action was calculated. Therefore, the objective function value, which corresponds in our case to the expected Entropy at the end of the planning horizon, was identical for all methods.}
  \end{center}
\end{table}

Next we compare our method to iSAM2 and rAMDL, including all one-time calculations, in two different scenarios. The first scenario uses a similar prior belief to the previous ones, while the second uses a different prior that includes obstacles. We use iSAM2 \cite{Kaess12ijrr} to incrementally update the square root of the information matrix, then use the updated $R$ matrix to calculate the reward as in section \ref{sssection: DetMethods}. Figure \ref{fig: BoundsSim} shows the bounds for the first simulated scenario, 2500 paths were evaluated without re-planning, where we can see  that about 85 percent of the sub-optimal paths can be safely pruned using the bounds. Table \ref{tab:SimMPVsRest} shows the total planning time for each method in the first scenario. We show separately the time it takes to recover the required entries from the prior covariance matrix in table \ref{tab:CovRecovery}: \emph{worst-case} is the time it takes to recover the full matrix, while the \emph{actual} is the time it takes to recover only the entries required for the states involved in the measurements for the set of all trajectories. Generally, the earlier in the robot's past trajectory the involved states are, the more of the covariance matrix entries we would need to recover, which is time consuming. However, in a scenario such as that, the time it would take to update the posterior R matrix needed for iSAM2 would be significant as well.  
Including the covariance recover, we can observe that our  method is about 45 percent faster than rAMDL, and about 65 percent faster than iSAM2.  

 \begin{figure} [!h] 
	\includegraphics[width=\columnwidth]{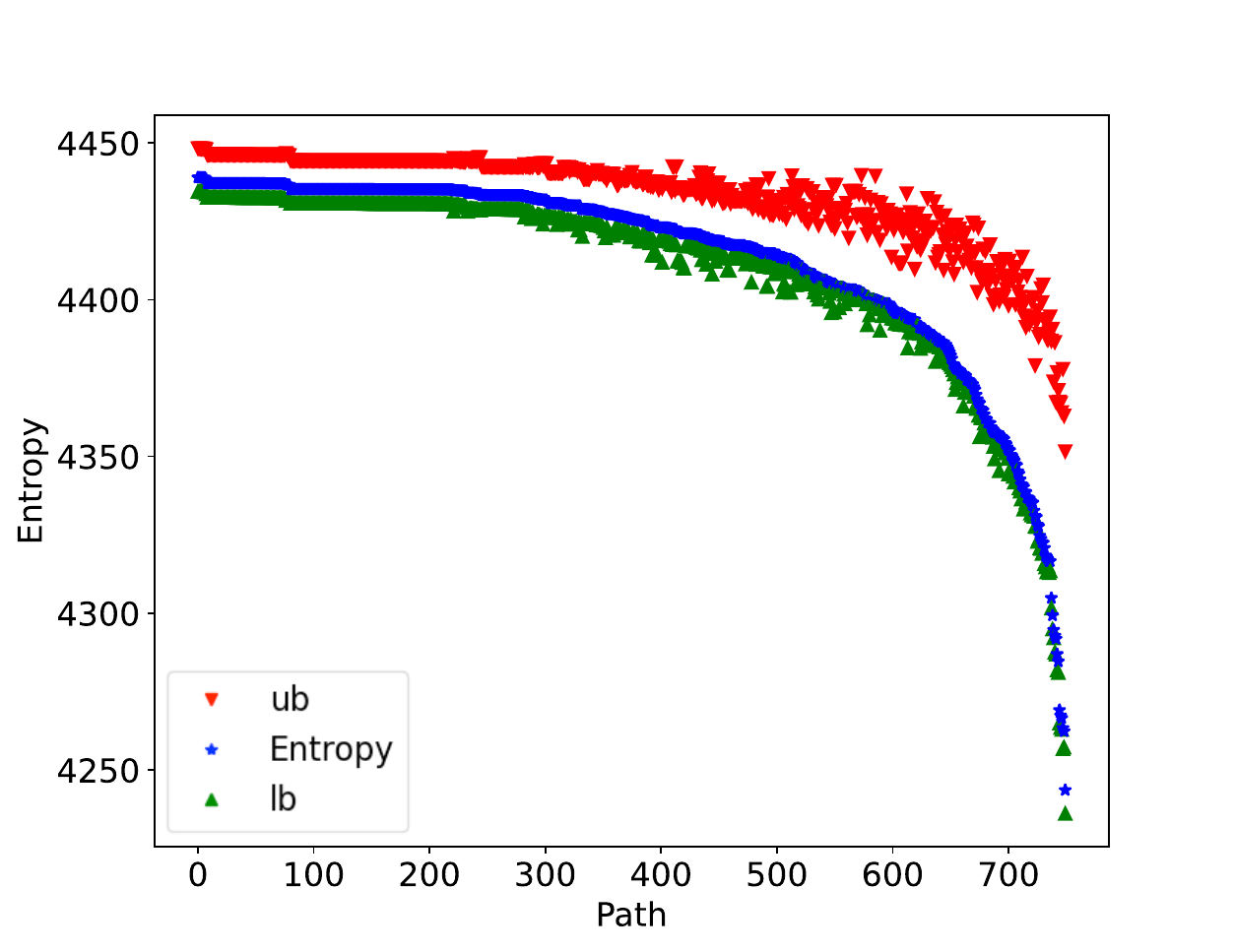}
	\scriptsize
	\caption{\scriptsize Bounds for the first simulated scenario, showing a subset of the lowest Entropy paths.}
	\normalsize
	\vspace{6pt}
	\label{fig: BoundsSim}
\end{figure} 

\begin{table}[!htb]
\vspace{20pt}
      
      \centering
\begin{tabular}[t]{c | c } 
Method & time [sec] \\
\hline
MP (ours) & $\boldsymbol{323.445 \pm 0.175}$  \\  
       rAMDL & ${590.584 \pm 0.463}$  \\
       iSAM2 & ${728.854 \pm 0.348}$ 

\end{tabular}
\caption{    \label{tab:SimMPVsRest}Total planning time in seconds (lower is better) for the first scenario, number of factors in the prior graph is 12185, number of paths evaluated is 2500. For all scenarios the same best action was calculated. Therefore, the objective function value was identical for all methods (4243.47).}
 \label{tab:SimMPVsRest}

\vspace{20pt}    
\begin{tabular}[t]{c|c} 

worst-case [sec]& actual [sec]\\
\hline
21.05 & 6.055 \\
\end{tabular}
\caption{Covariance recovery time.}
\label{tab:CovRecovery}
\end{table}

Introducing a second scenario, illustrated in Fig \ref{fig: IllustrationScenario2}, we compare the three methods again. The basic setting is similar, the only differences are the prior mapping, start and end points, and obstacles that were added to the map. The agent infers the location of the obstacles during the prior mapping session.  

 \begin{figure} [!h] 
	\includegraphics[width=\columnwidth]{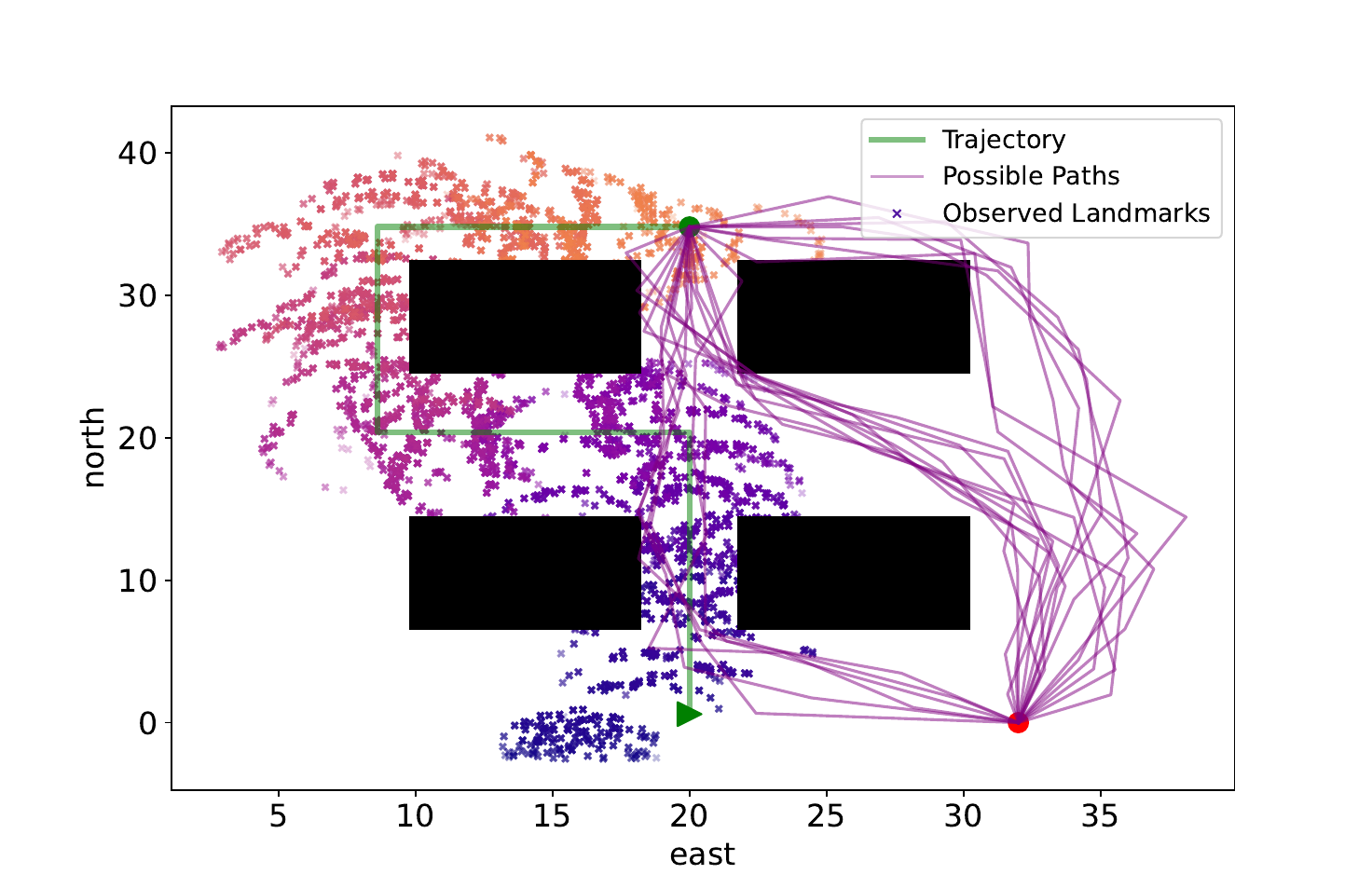}
	\scriptsize
	\caption{\scriptsize Illustration of the second simulated scenario, showing a subset of the landmarks and paths, obstacles shown in black. Each trajectory starts from the green dot and ends in the red one.}
	\normalsize
	\vspace{6pt}
	\label{fig: IllustrationScenario2}
\end{figure} 

Just like the previous scenario, the bounds are presented in Fig \ref{fig: BoundsSim2}, and Table \ref{tab:SimMPVsRest2} summarizes the total planning time for each method.
This scenario demonstrates a setting where there aren't many observations during planning. Fig.~\ref{fig: IllustrationScenario2} shows that most of the planning trajectories do not pass through previously mapped areas, hence, not generating future measurements, which explains what we see in Table \ref{tab:SimMPVsRest2}. Total planning time is shorter for all methods, which is caused by the small number of future measurements. We can see that our method is still the fastest without incurring any loss. However, given fewer future measurements, and as a consequence, the smaller number of Jacobian rows, the difference in planning time is not as drastic as in the previous scenario.  

 \begin{figure} [!h] 
	\includegraphics[width=\columnwidth]{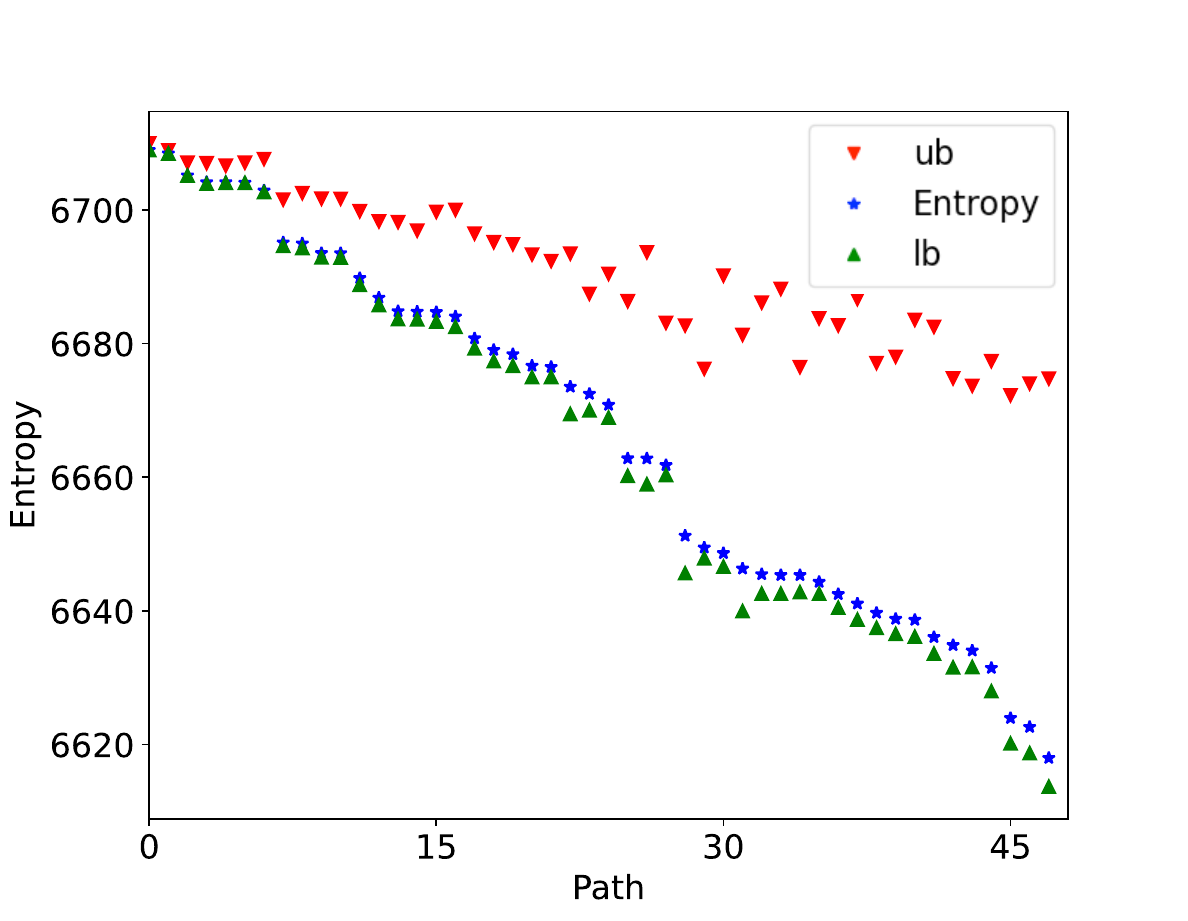}
	\scriptsize
	\caption{\scriptsize Bounds for second simulated scenario, showing a subset of the lowest Entropy paths.}
	\normalsize
	\vspace{6pt}
	\label{fig: BoundsSim2}
\end{figure} 

\begin{table}[!htb]
\vspace{20pt}
      
      \centering
\begin{tabular}[t]{c | c } 
Method & time [sec] \\
\hline
MP (ours) & $\boldsymbol{41.404 \pm  0.024}$  \\  
       rAMDL & ${50.186 \pm  0.037}$  \\
       iSAM2 & ${62.705 \pm 0.038}$ 

\end{tabular}
\caption{    \label{tab:SimMPVsRest2}Total planning time in seconds (lower is better) for the second scenario, number of factors in the prior graph is 17394, number of paths evaluated is 2500. For all scenarios the same best action was calculated. Therefore, the objective function value was identical for all methods (6617.20). }\label{tab:SimMPVsRest2}
\end{table}

\subsubsection{Sensitivity Study}
In this section we show the sensitivity of the different methods to the density of the prior information matrix. We use the same scenario from before, while adjusting the density of the prior information matrix by pruning factors from the prior factor graph. We use the number of landmarks connected to a certain pose (via a measurement factor) as a metric of density. This metric is then averaged across 300 evaluated paths. The results are presented in figure \ref{fig: SensitivityPriorDenseTotal}.

 \begin{figure} [!h] 
	\includegraphics[width=\columnwidth]{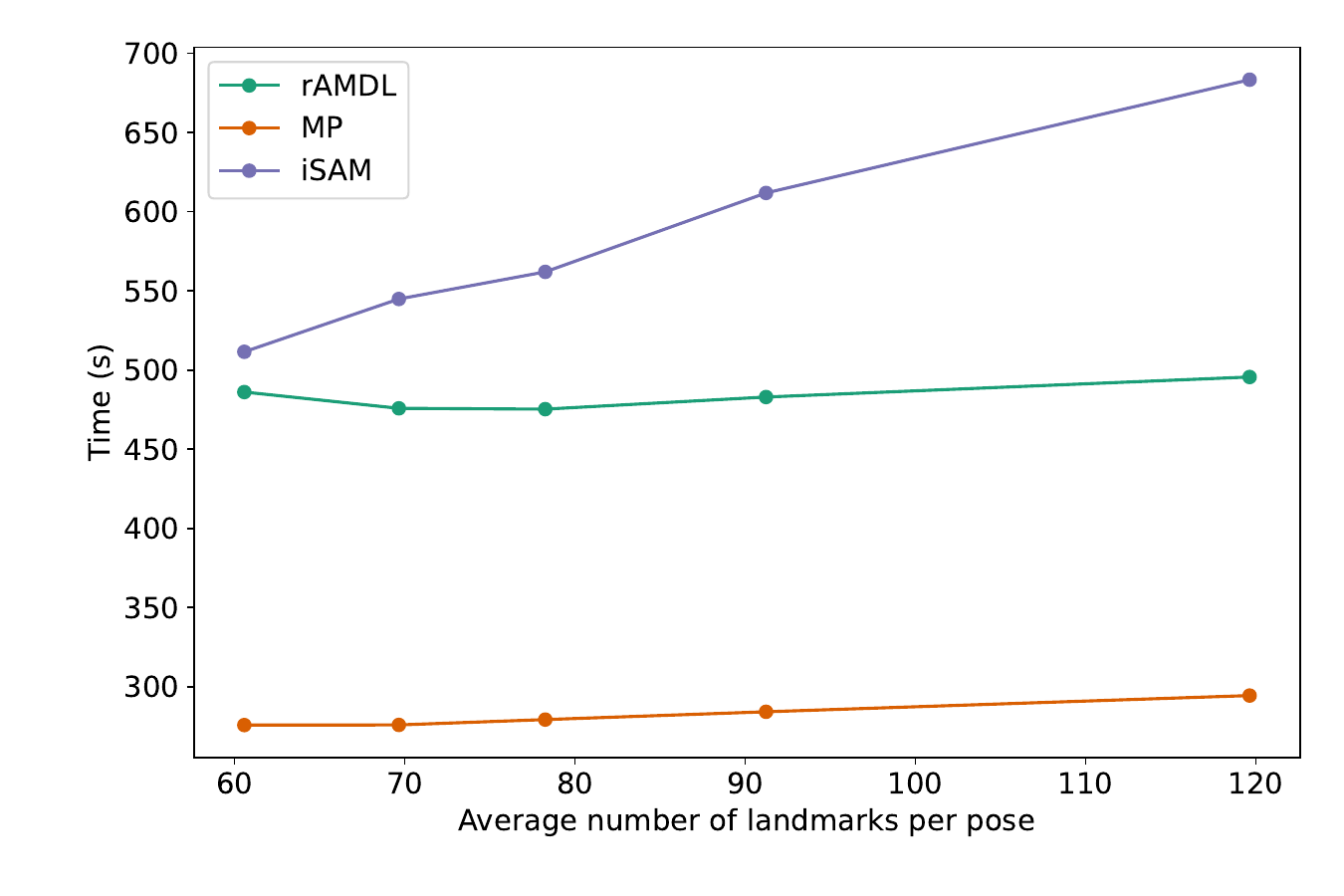}
	
	\caption{\scriptsize Comparison of total planning-time as a function of prior density.}
	\normalsize
	\vspace{6pt}
	\label{fig: SensitivityPriorDenseTotal}
\end{figure} 
We can observe that MP and rAMDL are largely un-affected by the change in prior density. This is expected since the complexity of both of them is a function of the Jacobian as explained in section \ref{sssection: GaussianCC}.
On the other hand, iSAM2 is much more sensitive to the prior density as its planning-time grows rapidly compared to the other methods. In order to further dissect the difference between the methods, we compare the iSAM2 planning time to the one-time calculation associated with MP and rAMDL shown in figure \ref{fig: SensitivityPriorDense}. The main components of the one-time are the recovery of the prior information and prior covariance matrices. 
 Theoretically, both the one-time and iSAM2 should be affected by the change in density. In practice, we can see that the one-time calculation is affected by the change in density, but when compared to iSAM2 it effectively becomes negligible. The reason for this gap is that the one-time calculation, as its name suggests, is done only once per a given prior belief. Conversely, the components of the iSAM2 calculation that are affected by the change in density have to be evaluated from scratch for each candidate path. When comparing the total planning-time of a pool of candidate paths, the growing cost associated with the one-time calculation as the density increases, becomes negligible. 
 \begin{figure} [!h] 
	\includegraphics[width=\columnwidth]{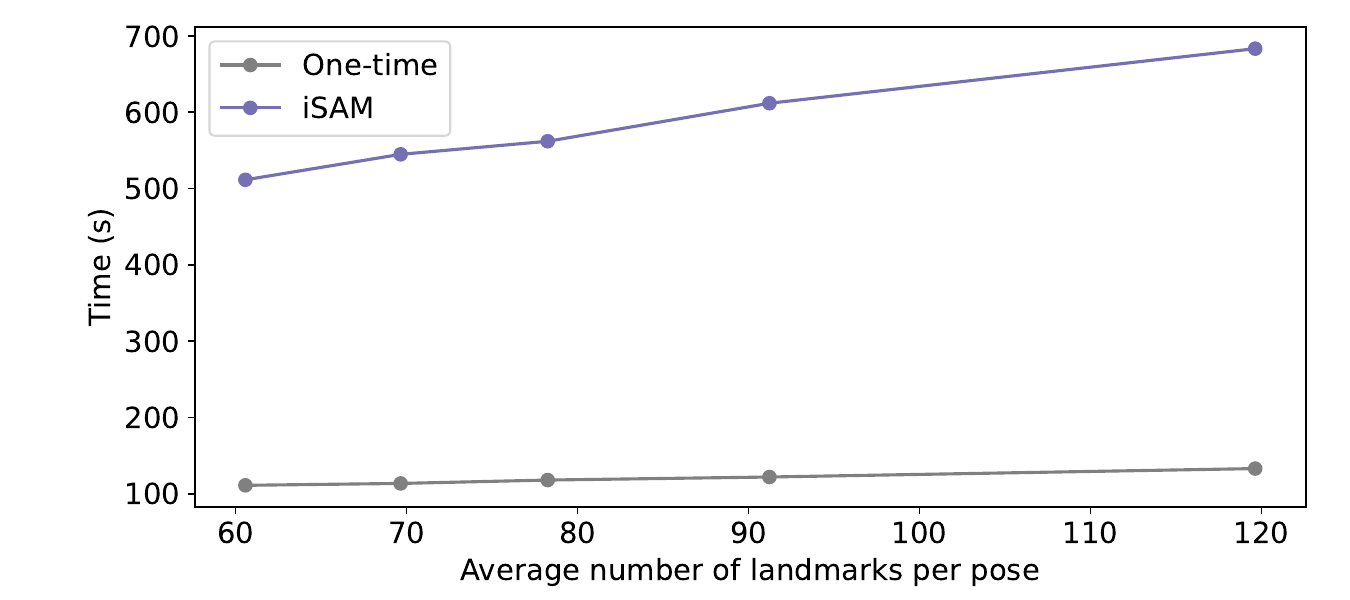}
	\scriptsize
	\caption{\scriptsize Comparison of planning-time as a function of prior density. iSAM2 planning-time grows as a function of the prior density, while the one-time calculation needed for MP stays flat.}
	\normalsize
	\vspace{6pt}
	\label{fig: SensitivityPriorDense}
\end{figure}

\subsubsection{Experiment}
Next, we demonstrate the use of the bounds in a real-world experiment using a DJI Robomaster S1, equipped with a ZED camera (used in a monocular configuration). GTSAM is used for the backend while the frontend is handled by SuperGlue \cite{Sarlin20cvpr}. The planning session starts after a partial mapping of an indoor environment,  which is stored as a prior factor-graph. From its current pose, the robot is given a goal, it then uses the PRM to generate possible trajectories to its goal. For each method 500 paths were evaluated, each consists of 20 actions.

Figure \ref{fig: Robomaster} shows the hardware setup for the experiment, figure \ref{fig: SuperGlue} shows a typical feature matching on two consecutive-in-time images. Visual odometry is used as motion model factors, and projection factors are used as observation factors.
 \begin{figure} [!h] 
 \centering
	\includegraphics[scale=0.09]{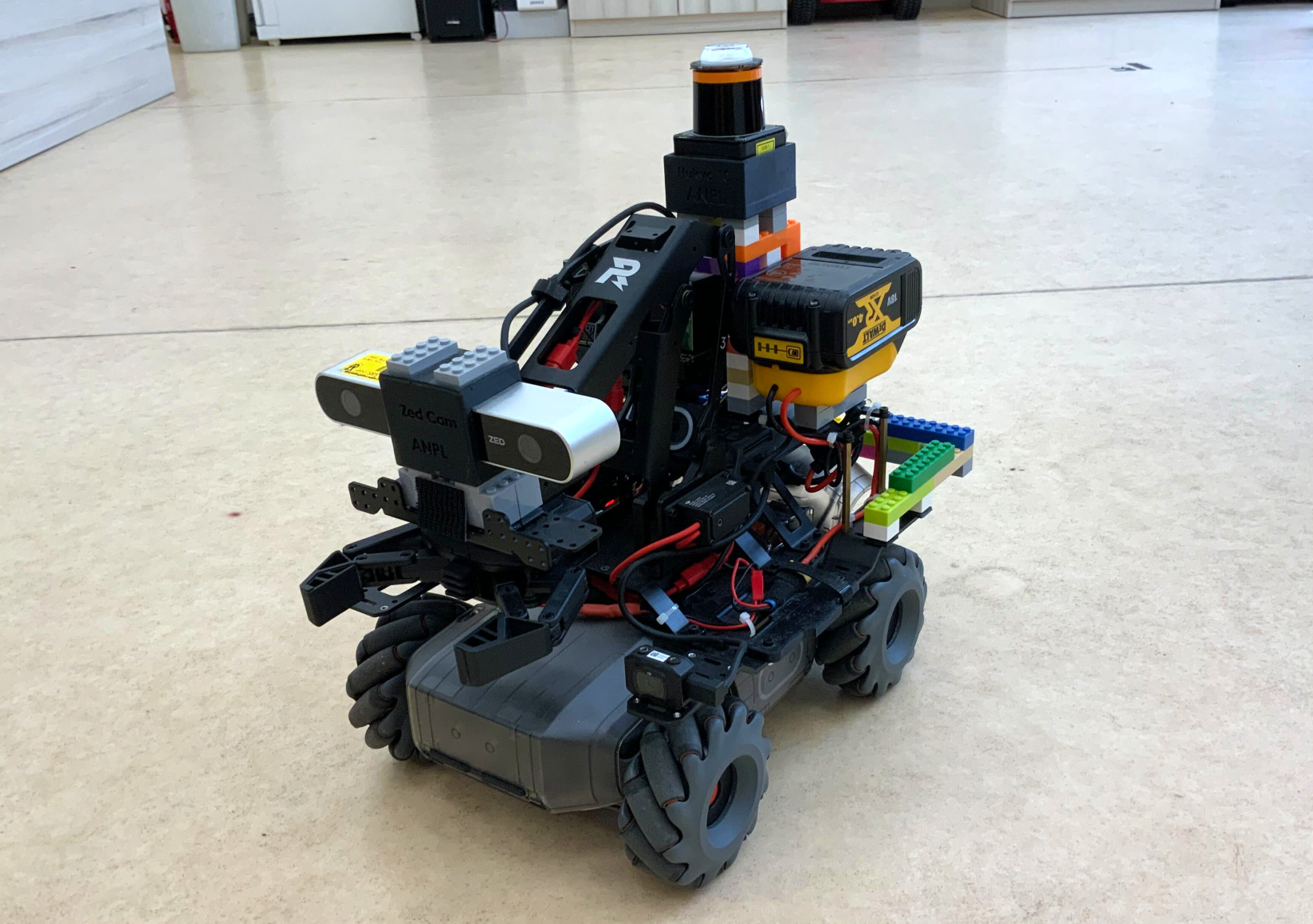}
	\scriptsize
	\caption{\scriptsize Robomaster S1 robot, equipped with a ZED camera. }
	\normalsize
	\vspace{6pt}
	\label{fig: Robomaster}
\end{figure} 

 \begin{figure} [!h] 
	\includegraphics[width=\columnwidth]{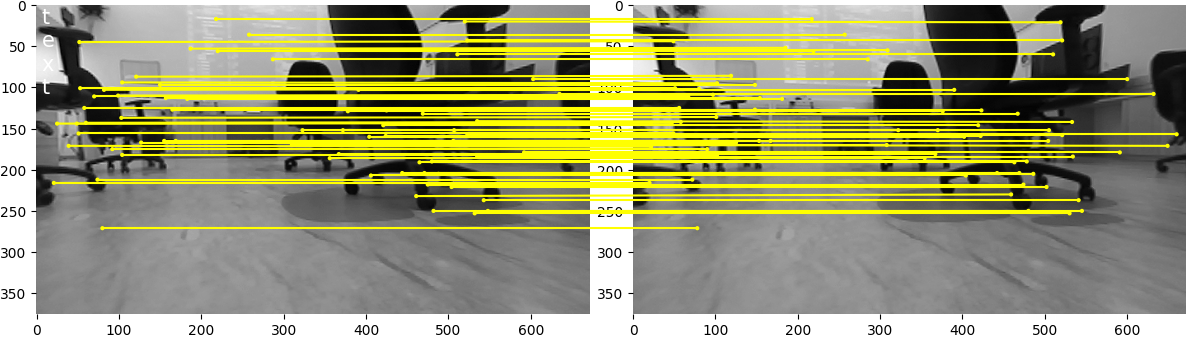}
	\scriptsize
	\caption{\scriptsize SuperGlue feature matching.}
	\normalsize
	\vspace{6pt}
	\label{fig: SuperGlue}
\end{figure}

Table \ref{tab: RealMPVsRest} summarizes the planning time for each of the methods and figure \ref{fig: ExperimentBounds} shows the expected reward bounds. 
We can see again that our method allows for a meaningful speed-up. About 70 percent of the paths can be pruned, while the lowest lower bound corresponds to the optimal expected reward.
The planning time of our method is about 30 percent faster than rAMDL, and about 70 percent faster than iSAM2. The larger difference between iSAM2 and the other two methods, compared to the same difference in section \ref{sssection: Simulation} can be explained by the sensitivity to the density of the prior shown in the previous section. The average number of landmarks per pose for the simulated scenario is 72.46, compared to 103.89 for the experiment. 

\begin{table}[h!]
  \begin{center}
    \begin{tabular}{ c | c  }     
       Method & time [sec] \\
      \hline
       MP (ours) & $\boldsymbol{585.05 \pm 84.78}$  \\  
       rAMDL & ${802.25 \pm 74.13}$  \\
       iSAM2 & ${1764.25 \pm 17.05}$ 
     
    \end{tabular}
    \caption{    \label{tab: RealMPVsRest}Total planning time in seconds (lower is better), number of factors in the prior graph is 8220, number of paths evaluated is 500. For all scenarios the same best action was calculated. Therefore, the objective function value was identical for all methods (-12884.88).}
  \end{center}
\end{table}

 \begin{figure} [!h] 
	\includegraphics[width=\columnwidth]{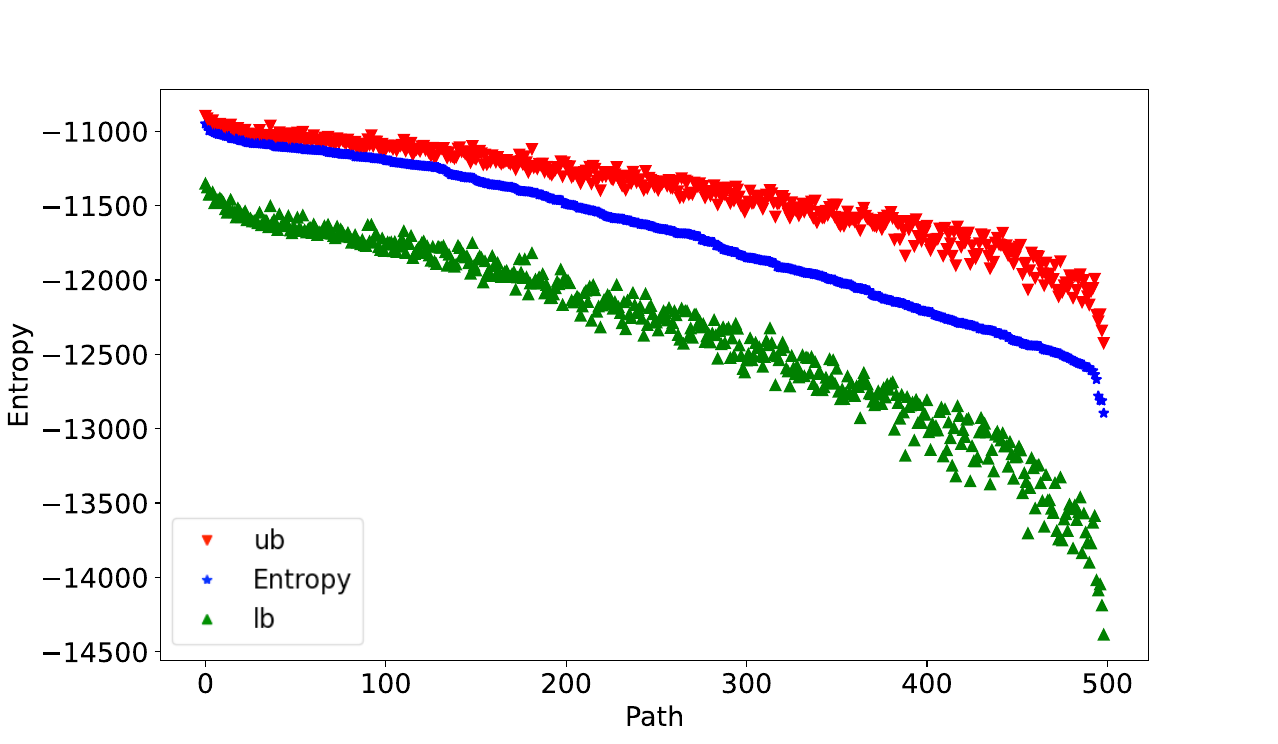}
	\scriptsize
	\caption{\scriptsize Robotmaster experiment.}
	\normalsize
	\vspace{6pt}
	\label{fig: ExperimentBounds}
\end{figure}

%% file: 06-conclusions.tex
\section{CONCLUSIONS} \label{sec: conclusions}
We introduced the novel concept of observation space partitioning for BSP and POMDP planning problems. The concept is general and applies to all belief distributions and the underlying POMDP spaces. It allows for a more efficient way of identifying the optimal action by using a simplification paradigm and forming analytical bounds on the expected sum of rewards.
We have demonstrated one possible use-case of this concept by studying a typical active SLAM scenario with Gaussian beliefs. We have shown that both for simulated and real-world experiments, our method has a faster planning running time when compared to other state of the art methods, while identifying the same optimal trajectory.

\subsection{Limitations and Future Work}
While the theoretical foundation of the proposed method is general, its implementation has some limitations. Tackling such limitations can form the basis for future research. The method presented is applicable for information-theoretic rewards (specifically, Entropy), which are typically the computational bottleneck compared to state-dependent rewards. It is not hard however, to extend this approach to other information-theoretic rewards such as Information Gain.    
   Additionally, future work can evolve in multiple different directions, which might include; implementations for non-parametric belief distributions, an extension to policies over action-sequences, and the possible partitioning of other POMDP spaces.

%% file: 08-acknowledgment.tex
\section{ACKNOWLEDGMENT}
The authors would like to thank Anton Gulyaev for his  help with the real-world experiments and the sensitivity study. %\\
%The research presented in this paper was partially funded by the Israeli Smart Transportation Research Center (ISTRC). 

%% file: MS-BSP.bbl
\begin{thebibliography}{10}

\bibitem{Araya10nips}
Mauricio Araya, Olivier Buffet, Vincent Thomas, and Fran{\c{c}}cois Charpillet.
\newblock A pomdp extension with belief-dependent rewards.
\newblock In {\em Advances in Neural Information Processing Systems (NIPS)},
  pages 64--72, 2010.

\bibitem{Barenboim22ijcai}
M.~Barenboim and V.~Indelman.
\newblock Adaptive information belief space planning.
\newblock In {\em the 31st International Joint Conference on Artificial
  Intelligence and the 25th European Conference on Artificial Intelligence
  (IJCAI-ECAI)}, July 2022.

\bibitem{Boers10fusion}
Y.~{Boers}, H.~{Driessen}, A.~{Bagchi}, and P.~{Mandal}.
\newblock Particle filter based entropy.
\newblock In {\em 2010 13th International Conference on Information Fusion},
  pages 1--8, 2010.

\bibitem{CarlevarisBianco14tro}
Nicholas Carlevaris-Bianco, Michael Kaess, and Ryan~M Eustice.
\newblock Generic node removal for factor-graph {SLAM}.
\newblock {\em {IEEE} Trans. Robotics}, 30(6):1371--1385, 2014.

\bibitem{Davison05iccv}
A.J. Davison.
\newblock Active search for real-time vision.
\newblock In {\em Intl. Conf. on Computer Vision (ICCV)}, pages 66--73, Oct
  2005.

\bibitem{Elimelech22ijrr}
Khen Elimelech and Vadim Indelman.
\newblock Simplified decision making in the belief space using belief
  sparsification.
\newblock {\em The International Journal of Robotics Research}, 41(5):470--496,
  2022.

\bibitem{Indelman16ral}
V.~Indelman.
\newblock No correlations involved: Decision making under uncertainty in a
  conservative sparse information space.
\newblock {\em IEEE Robotics and Automation Letters (RA-L)}, 1(1):407--414,
  2016.

\bibitem{Indelman15ijrr}
V.~Indelman, L.~Carlone, and F.~Dellaert.
\newblock Planning in the continuous domain: a generalized belief space
  approach for autonomous navigation in unknown environments.
\newblock {\em Intl. J. of Robotics Research}, 34(7):849--882, 2015.

\bibitem{Kaelbling98ai}
L.~P. Kaelbling, M.~L. Littman, and A.~R. Cassandra.
\newblock Planning and acting in partially observable stochastic domains.
\newblock {\em Artificial intelligence}, 101(1):99--134, 1998.

\bibitem{Kaess12ijrr}
M.~Kaess, H.~Johannsson, R.~Roberts, V.~Ila, J.~Leonard, and F.~Dellaert.
\newblock {iSAM2}: Incremental smoothing and mapping using the {B}ayes tree.
\newblock {\em Intl. J. of Robotics Research}, 31(2):217--236, Feb 2012.

\bibitem{Kaess08tro}
M.~Kaess, A.~Ranganathan, and F.~Dellaert.
\newblock {iSAM}: Incremental smoothing and mapping.
\newblock {\em {IEEE} Trans. Robotics}, 24(6):1365--1378, Dec 2008.

\bibitem{Kavraki96tra}
L.E. Kavraki, P.~Svestka, J.-C. Latombe, and M.H. Overmars.
\newblock Probabilistic roadmaps for path planning in high-dimensional
  configuration spaces.
\newblock {\em {IEEE} Trans. Robot. Automat.}, 12(4):566--580, 1996.

\bibitem{Khosoussi20SPAR}
Kasra Khosoussi, Gaurav~S. Sukhatme, Shoudong Huang, and Gamini Dissanayake.
\newblock {\em Designing Sparse Reliable Pose-Graph SLAM: A Graph-Theoretic
  Approach}, pages 17--32.
\newblock Springer International Publishing, 2020.

\bibitem{Kim14ijrr}
A.~Kim and R.~M. Eustice.
\newblock Active visual {SLAM} for robotic area coverage: Theory and
  experiment.
\newblock {\em Intl. J. of Robotics Research}, 34(4-5):457--475, 2014.

\bibitem{Kopitkov17ijrr}
D.~Kopitkov and V.~Indelman.
\newblock No belief propagation required: Belief space planning in
  high-dimensional state spaces via factor graphs, matrix determinant lemma and
  re-use of calculation.
\newblock {\em Intl. J. of Robotics Research}, 36(10):1088--1130, August 2017.

\bibitem{Kopitkov19ijrr}
Dmitry Kopitkov and Vadim Indelman.
\newblock General purpose incremental covariance update and efficient belief
  space planning via factor-graph propagation action tree.
\newblock {\em Intl. J. of Robotics Research}, 38(14):1644--1673, 2019.

\bibitem{Krause08jmlr}
A.~Krause, A.~Singh, and C.~Guestrin.
\newblock Near-optimal sensor placements in gaussian processes: Theory,
  efficient algorithms and empirical studies.
\newblock {\em J. of Machine Learning Research}, 9:235--284, 2008.

\bibitem{Kretzschmar12ijrr}
H.~Kretzschmar and C.~Stachniss.
\newblock Information-theoretic compression of pose graphs for laser-based
  {SLAM}.
\newblock {\em Intl. J. of Robotics Research}, 31(11):1219--1230, 2012.

\bibitem{Papadimitriou87math}
C.~Papadimitriou and J.~Tsitsiklis.
\newblock The complexity of {Markov} decision processes.
\newblock {\em Mathematics of operations research}, 12(3):441--450, 1987.

\bibitem{Platt10rss}
R.~Platt, R.~Tedrake, L.P. Kaelbling, and T.~Lozano-P\'erez.
\newblock Belief space planning assuming maximum likelihood observations.
\newblock In {\em Robotics: Science and Systems (RSS)}, pages 587--593,
  Zaragoza, Spain, 2010.

\bibitem{Popovic17icra}
Marija Popovic, Gregory Hitz, Juan Nieto, Inkyu Sa, Roland Siegwart, and Enric
  Galceran.
\newblock Online informative path planning for active classification using
  uavs.
\newblock In {\em IEEE Intl. Conf. on Robotics and Automation (ICRA)}, 2017.

\bibitem{Roy05jair}
Nicholas Roy, Geoffrey~J Gordon, and Sebastian Thrun.
\newblock Finding approximate pomdp solutions through belief compression.
\newblock {\em J. Artif. Intell. Res.(JAIR)}, 23:1--40, 2005.

\bibitem{Sarlin20cvpr}
Paul-Edouard Sarlin, Daniel DeTone, Tomasz Malisiewicz, and Andrew Rabinovich.
\newblock Superglue: Learning feature matching with graph neural networks.
\newblock In {\em IEEE Conf. on Computer Vision and Pattern Recognition
  (CVPR)}, pages 4938--4947, 2020.

\bibitem{Smith04uai}
T.~Smith and R.~Simmons.
\newblock Heuristic search value iteration for pomdps.
\newblock In {\em Conf. on Uncertainty in Artificial Intelligence (UAI)}, pages
  520--527, 2004.

\bibitem{Somani13nips}
Adhiraj Somani, Nan Ye, David Hsu, and Wee~Sun Lee.
\newblock Despot: Online pomdp planning with regularization.
\newblock In {\em NIPS}, volume~13, pages 1772--1780, 2013.

\bibitem{Stachniss05rss}
C.~Stachniss, G.~Grisetti, and W.~Burgard.
\newblock Information gain-based exploration using {Rao-Blackwellized} particle
  filters.
\newblock In {\em Robotics: Science and Systems (RSS)}, pages 65--72, 2005.

\bibitem{Sunberg18icaps}
Zachary Sunberg and Mykel Kochenderfer.
\newblock Online algorithms for pomdps with continuous state, action, and
  observation spaces.
\newblock In {\em Proceedings of the International Conference on Automated
  Planning and Scheduling}, volume~28, 2018.

\bibitem{Sztyglic22iros}
Ori Sztyglic and Vadim Indelman.
\newblock Speeding up online pomdp planning via simplification.
\newblock In {\em IEEE/RSJ Intl. Conf. on Intelligent Robots and Systems
  (IROS)}, 2022.

\bibitem{VanDenBerg12ijrr}
J.~Van Den~Berg, S.~Patil, and R.~Alterovitz.
\newblock Motion planning under uncertainty using iterative local optimization
  in belief space.
\newblock {\em Intl. J. of Robotics Research}, 31(11):1263--1278, 2012.

\bibitem{Zhang15cvpr}
Guangcong Zhang and Patricio~A. Vela.
\newblock Good features to track for visual slam.
\newblock In {\em 2015 IEEE Conference on Computer Vision and Pattern
  Recognition (CVPR)}, pages 1373--1382, 2015.

\bibitem{Zhitnikov22ai}
A.~Zhitnikov and V.~Indelman.
\newblock Simplified risk aware decision making with belief dependent rewards
  in partially observable domains.
\newblock {\em Artificial Intelligence, Special Issue on ``Risk-Aware
  Autonomous Systems: Theory and Practice"}, 2022.

\bibitem{Zhitnikov24ijrr_accepted}
A.~Zhitnikov, O.~Sztyglic, and V.~Indelman.
\newblock No compromise in solution quality: Speeding up belief-dependent
  continuous pomdps via adaptive multilevel simplification.
\newblock {\em International Journal of Robotics Research (IJRR), accepted},
  2024.

\end{thebibliography}
